\documentclass{article}

 \usepackage[preprint]{neurips_2026}

\usepackage[utf8]{inputenc} 
\usepackage[T1]{fontenc}    
\usepackage{hyperref}       
\usepackage{url}            
\usepackage{booktabs}       
\usepackage{amsfonts}       
\usepackage{nicefrac}       
\usepackage{microtype}      
\usepackage{xcolor}         
\usepackage{amsmath, amssymb, amsthm, bm}
\usepackage{hyperref}
\usepackage{graphicx}
\usepackage{caption}
\usepackage{subcaption}
\usepackage{multirow}
\usepackage{multicol}
\usepackage[export]{adjustbox} 
\usepackage{algorithm}
\usepackage{algorithmic}
\usepackage{diagbox}
\newtheorem{proposition}{Proposition}
\newtheorem{definition}{Definition}
\theoremstyle{remark}
\newtheorem{remark}{Remark}
\newtheorem*{proposition*}{Proposition}
\theoremstyle{plain}
\newtheorem*{propositionB}{Proposition}

\title{Conditional Compatibility Learning for Context-Dependent Anomaly Detection}

\author{%
  Shashank Mishra \\
  German Research Center for Artificial Intelligence (DFKI) \\
  Kaiserslautern, Germany \\
  \texttt{shashank.mishra@dfki.de}
  \And
  Didier Stricker \\
  DFKI and RPTU Kaiserslautern \\
  Kaiserslautern, Germany \\
  \texttt{didier.stricker@dfki.de}
  \And
  Jason Rambach \\
  German Research Center for Artificial Intelligence (DFKI) \\
  Kaiserslautern, Germany \\
  \texttt{jason.rambach@dfki.de}
}

\begin{document}

\maketitle

\begin{abstract}
Anomaly detection usually assumes that abnormality is an intrinsic property of an observation. A defect is a defect, and a rare object is rare, regardless of where it appears. Many real-world anomalies do not work this way. A runner on a track is normal, but the same runner on a highway is not. The subject is unchanged; only the context makes it anomalous. This setting, long recognized as contextual anomaly detection, remains largely underexplored in modern vision–language systems. The difficulty is not merely empirical, it is formal. When anomaly labels depend on the relation between a subject and its context, any detector reasoning from a global representation that conflates subject and context is provably non-identifiable: two different subject–context configurations can map to the same embedding while requiring opposite labels, and no such detector can be correct on both. This impossibility motivates a different formulation: instead of asking whether an observation deviates from a global notion of normality, the model should ask whether subjects are compatible with their surrounding context. We define this as conditional compatibility learning. We instantiate this framework in CC-CLIP, a vision–language architecture that learns disentangled subject- and context-aware representations from a single image and fuses visual evidence through text-conditioned attention. CC-CLIP achieves state-of-the-art results on real-world contextual anomaly detection, substantially outperforming all existing CLIP-based and context-reasoning baselines. A single-branch variant of CC-CLIP also achieves competitive performance on structural anomaly benchmarks.
\end{abstract}

\section{Introduction}

Anomaly detection in computer vision is typically formulated as identifying observations that deviate from the distribution of normal data. Underlying this formulation is a strong assumption: that abnormality is an intrinsic property of the observation itself. A scratched surface is always defective. A cracked tile is always anomalous. This assumption has driven a decade of progress on benchmarks such as MVTec-AD~\citep{bergmann2019mvtec} and VisA~\citep{zou2022spot}, where anomalies correspond to texture defects, surface damage, or structural irregularities that can be identified from the object alone, independent of where or how it appears.

However, many real-world anomalies are not intrinsic. Consider a person running: entirely normal on a jogging track, but anomalous on a highway. Consider children playing with a ball: normal in a park, but anomalous on a busy street. In these cases, the subject itself is visually ordinary and individually common. Nothing about the runner or the children, viewed in isolation, signals abnormality. The anomaly arises entirely from the incompatibility between the subject and its surrounding context. This form of context-dependent abnormality has long been recognized in the classical anomaly detection literature as \emph{contextual anomaly detection}~\citep{chandola2009anomaly}, and has been studied in object-centric settings under the label of out-of-context detection~\citep{choi2012context, bomatter2021pigs, acharya2022detecting}. Yet it remains largely underexplored in modern vision--language systems, where methods are overwhelmingly optimized for structural or defect-level anomalies and assume that a single global representation of the image is sufficient for detection.

This gap is not merely empirical. We show that it is a fundamental limitation of how representations are constructed. When anomaly labels depend on the relationship between a subject and its context, two visually distinct scenes can map to the same global embedding while requiring opposite labels. No detector built on such a representation can be correct on both. We formalize this as a non-identifiability result (Proposition~\ref{prop:nonidentifiability}, Section~\ref{sec:formulation}) and argue that it explains why existing CLIP-based anomaly detectors \citep{ma2025aaclip,deng2022anomaly,jeong2023winclip,cao2024adaclip,li2024promptad}, despite their strong performance on structural benchmarks, consistently fail on contextual anomalies. Figure~\ref{fig:teaser} illustrates this empirically: CLIP embeddings entangle contextually normal and anomalous observations into overlapping clusters, making them indistinguishable to any downstream classifier.

This inherent limitation motivates a fundamentally different formulation. Instead of asking whether an observation deviates from a global notion of normality, the model should ask whether a subject is compatible with its surrounding context. We define this as \emph{conditional compatibility learning}. Under this formulation, the model must learn to decompose an image into complementary representations that separately capture \emph{what} is present and \emph{where} it appears, reason about their relationship, and generalize to unseen combinations. We instantiate this framework in \textbf{CC-CLIP} (Conditional Compatibility CLIP), a vision--language architecture that achieves this decomposition through text-conditioned adaptation of a frozen CLIP backbone. Paired normal and anomalous text anchors, refined via a disentanglement objective, guide specialization through gradient signals alone. A text-conditioned attention module then fuses the resulting representations to produce a compatibility-based anomaly score.

Our contributions are as follows:
\begin{itemize}
    \item \textbf{Conceptual.} We formalize contextual anomaly detection in the vision--language setting and prove a non-identifiability result showing that global representations are provably insufficient when anomaly labels depend on subject--context compatibility.
    \item \textbf{Methodological.} We propose conditional compatibility learning and instantiate it in CC-CLIP, a CLIP-based architecture that achieves text-conditioned decomposition from a single image for context-aware anomaly reasoning.
    \item \textbf{Empirical.} CC-CLIP achieves state-of-the-art performance on all real-world contextual anomaly detection datasets, substantially outperforming all existing CLIP-based and context-reasoning baselines despite being trained entirely on synthetic data. A single-branch variant also attains competitive results on structural anomaly benchmarks, showing no loss in generality.
\end{itemize}

\begin{figure}[t]
\centering
\includegraphics[width=0.95\linewidth]{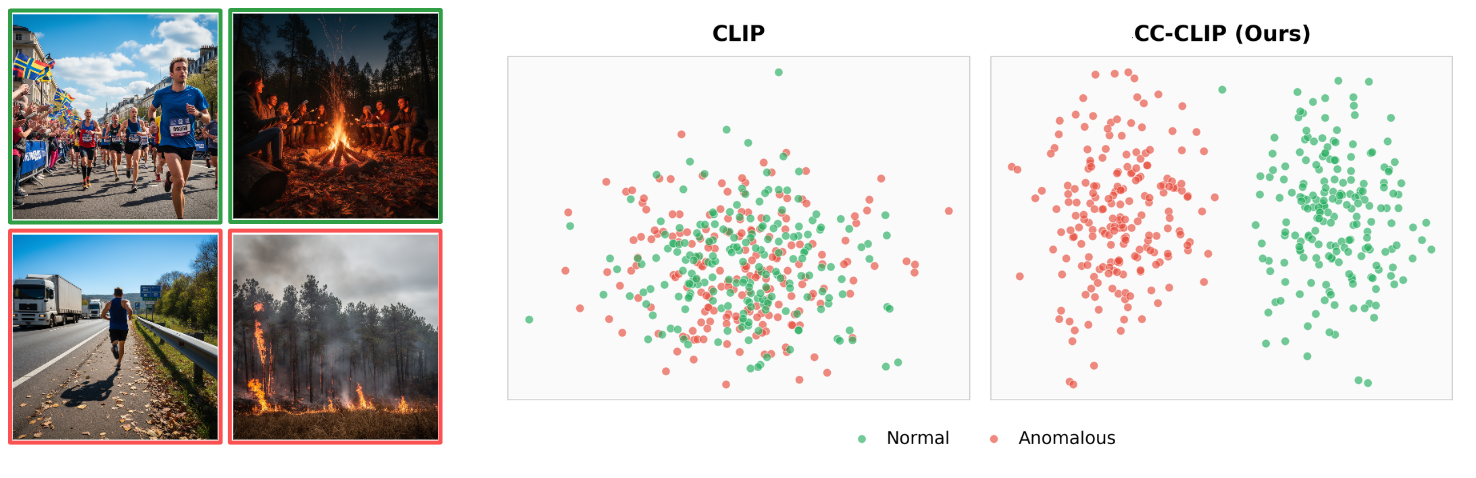}
\caption{Contextual anomalies arise from subject--context incompatibility, not intrinsic appearance. \textbf{Left:} The same subject is normal (\textcolor{green!60!black}{\rule{6pt}{6pt}}) or anomalous (\textcolor{red!80!black}{\rule{6pt}{6pt}}) depending solely on the surrounding context. \textbf{Right:} t-SNE of 1,026 real-world images. CLIP entangles contextually normal and anomalous observations; CC-CLIP learns separated representations, confirming the non-identifiability formalized in Proposition~\ref{prop:nonidentifiability}.}
\label{fig:teaser}
\end{figure}

\section{Related Work}

\textbf{Vision--language anomaly detection.}
Recent anomaly detection methods increasingly leverage pretrained vision--language models such as CLIP~\citep{radford2021learning} for open-vocabulary inference. WinCLIP~\citep{jeong2023winclip} introduces prompt ensembling for zero-shot defect detection, while AnomalyCLIP~\citep{zhou2023anomalyclip} and AdaCLIP~\citep{cao2024adaclip} adapt CLIP via prompt learning for few-shot settings. PromptAD~\citep{li2024promptad} and IIPAD~\citep{lv2025oneforall} incorporate iterative refinement, and AA-CLIP~\citep{ma2025aaclip} injects anomaly-aware supervision into the text encoder. Despite strong performance on structural benchmarks such as MVTec-AD~\citep{bergmann2019mvtec} and VisA~\citep{zou2022spot}, all these methods treat anomaly detection as an object-centric problem where anomalies correspond to local defects or appearance deviations and do not model whether a visually normal subject is compatible with its surrounding context.

\textbf{Contextual anomaly and out-of-context detection.}
Contextual anomaly detection, where anomaly labels depend on the relationship between an observation and its context, is a well-established concept in the classical anomaly detection taxonomy~\citep{chandola2009anomaly}. In computer vision, related work on 
out-of-context (OOC) detection includes graph-based reasoning~\citep{bomatter2021pigs, 
acharya2022detecting} and foundation-model prompting~\citep{roy2025zeroshot}. However, existing OOC benchmarks such as MIT-OOC~\citep{choi2012context} and COCO-OOC~\citep{acharya2022detecting} introduce anomalies through physical or spatial implausibility that can often be detected via low-level visual cues without semantic reasoning; graph-based methods~\citep{bomatter2021pigs, acharya2022detecting} similarly 
reason over spatial object relations rather than text-conditioned semantic compatibility. In contrast, contextual anomaly detection as studied in this work requires recognizing that two individually plausible components are semantically incompatible, demanding relational reasoning beyond appearance-based detection. Unlike video anomaly 
detection~\citep{sultani2018real, acsintoae2022ubnormal}, which identifies anomalies through temporal irregularities, our setting is orthogonal, concerning static semantic subject--context compatibility without temporal information.


\section{Problem Formulation}
\label{sec:formulation}

\textbf{Setup.}
Let $x \in \mathcal{X}$ denote an observed image whose semantic content is determined by a set of latent subjects $\mathcal{S} = \{a_1, \ldots, a_n\}$ where $a_i \in \mathcal{A}$, and a surrounding context $c \in \mathcal{C}$, formally captured by a mapping $x = g(\mathcal{S}, c)$. The anomaly label is determined by a compatibility function
\begin{equation}
    y = h(\mathcal{S}, c) \in \{0, 1\},
\end{equation}
where $y = 1$ indicates subject--context incompatibility. Crucially, $y$ depends on the \emph{relation} between $\mathcal{S}$ and $c$: the same subject may be normal in one context and anomalous in another.

\textbf{Non-identifiability of global representations.}
Standard anomaly detectors operate on a single representation of the full image. We show that this is provably insufficient under context-dependent labels when such representations conflate subject and context, a condition systematically induced by contrastive pretraining. For clarity of exposition, the proposition is stated for a single subject $a \in \mathcal{A}$; the argument extends analogously to $|\mathcal{S}| > 1$.

\begin{proposition}[Non-identifiability under intrinsic representations]
\label{prop:nonidentifiability}
Let $\phi: \mathcal{X} \to \mathcal{Z}$ be any representation mapping and let $f(x) = \psi(\phi(x))$ be a detector for a decision function $\psi: \mathcal{Z} \to \{0,1\}$. If there exists a subject $a \in \mathcal{A}$ and contexts $c, c' \in \mathcal{C}$ such that
\begin{equation}
    \phi(g(a,c)) = \phi(g(a,c')) \quad \text{but} \quad h(a,c) \neq h(a,c'),
\end{equation}
then $f$ cannot be correct on both $x = g(a,c)$ and $x' = g(a,c')$.
\end{proposition}

\begin{proof}
Since $\phi(x) = \phi(x')$, we have $f(x) = \psi(\phi(x)) = \psi(\phi(x')) = f(x')$. But correctness requires $f(x) = h(a,c)$ and $f(x') = h(a,c')$, which yields $h(a,c) = h(a,c')$, contradicting the assumption.
\end{proof}
\noindent A detailed proof with formal definitions of all mappings is provided in Appendix~\ref{app:proof}.

Proposition~\ref{prop:nonidentifiability} formalizes the failure mode illustrated in Figure~\ref{fig:teaser}: when a representation $\phi$ maps two images of the same subject in different contexts to nearby embeddings, no downstream classifier can distinguish their labels. This is not a failure of a particular model but a fundamental limitation of detectors whose global representations conflate subject and context.


\textbf{Learning requirements.}
Proposition~\ref{prop:nonidentifiability} implies that effective contextual anomaly detection methods must satisfy:
\begin{itemize}
    \item \emph{Relational reasoning.} The model must represent subject and context separately rather than through a single entangled embedding.
    \item \emph{Context-sensitive discrimination.} Predictions must change when context changes while the subject is held fixed.
    \item \emph{Cross-context generalization.} The model must transfer to unseen subject--context combinations by learning semantic compatibility rather than memorizing specific pairings.
\end{itemize}
These requirements directly motivate the architecture introduced in Section~\ref{sec:method}.

\section{Method: CC-CLIP}
\label{sec:method}

\subsection{Overview}

\begin{figure}[t]
\centering
\includegraphics[width=\linewidth]{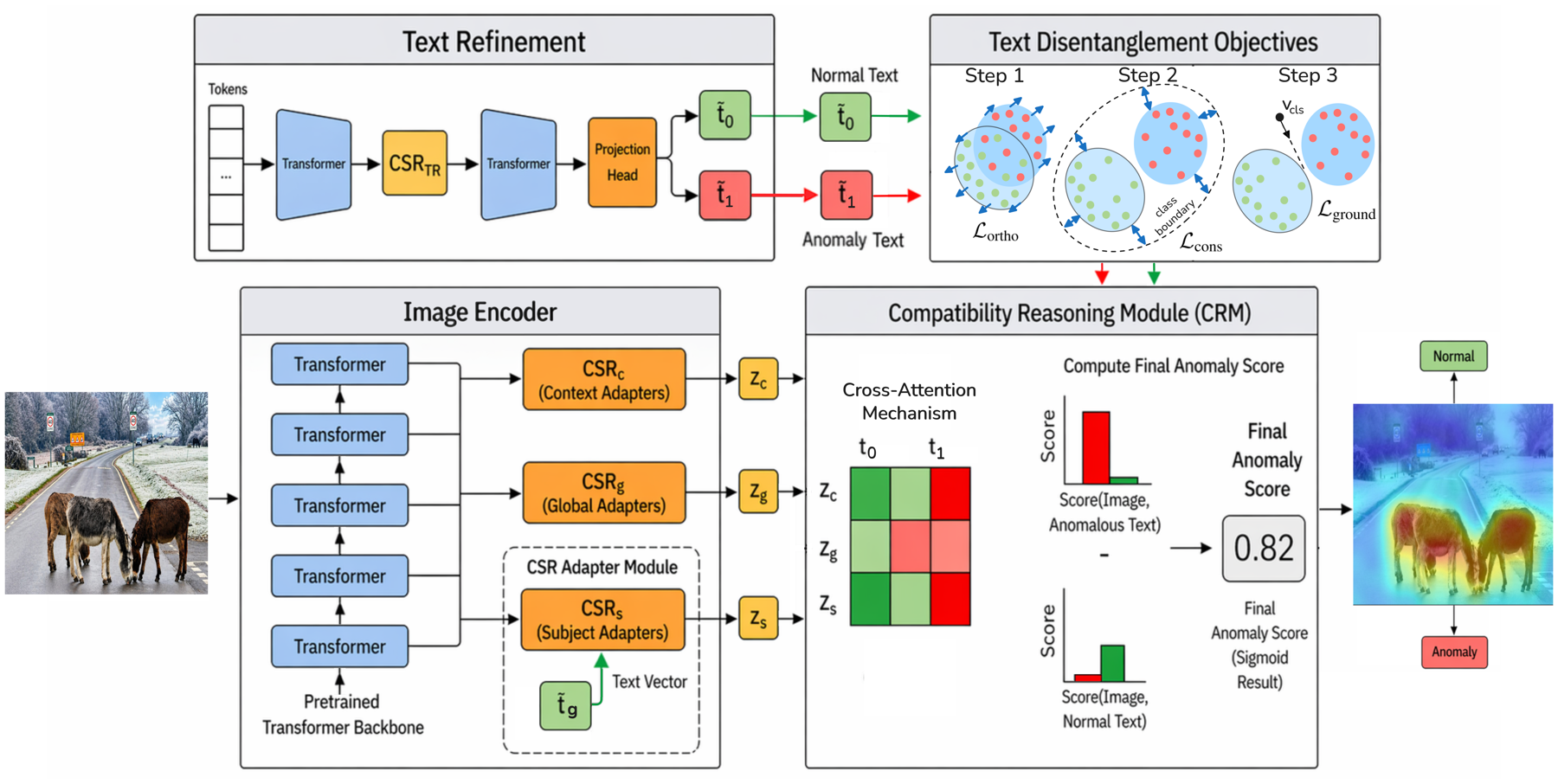}
\caption{CC-CLIP framework overview. Text refinement produces paired normal/anomalous anchors, while three text-guided CSR adapters decompose the input image into subject, context, and global representations, fused via text-conditioned cross-attention in the CRM to produce a compatibility-based anomaly score.}
\label{fig:architecture}
\end{figure}

Proposition~\ref{prop:nonidentifiability} motivates decomposed subject and context representations as a principled solution for contextual anomaly detection. CC-CLIP achieves this through text-conditioned specialization of a frozen CLIP backbone. Given an input image $x$, three lightweight adapter pathways process $x$ in parallel, each guided by a distinct textual role. A text refinement module produces paired normal and anomalous anchors, and a fusion module aggregates visual streams via text-conditioned attention into a compatibility-based anomaly score. The full pipeline is summarized in 
Algorithm~\ref{alg:cococlip} and Figure~\ref{fig:architecture}.



\begin{algorithm}[!ht]
\caption{CC-CLIP: Forward Pass and Training Procedure (two-stage; see Appendix~\ref{app:training}).}
\label{alg:cococlip}
\begin{algorithmic}[1]
\REQUIRE Image $x$, class name $\texttt{cls}$, label $y \in \{0,1\}$
\STATE \textbf{Branch-specific text prompts:}
\STATE \quad $T_s \leftarrow$ \texttt{"a photo of [cls]"}
\STATE \quad $T_c \leftarrow$ \texttt{"a scene with [cls] surroundings"}
\STATE \quad $T_g \leftarrow$ \texttt{"a photo of [cls] in a [context]"}
\STATE \textbf{Visual encoding} (shared frozen CLIP + branch-specific CSR):
\STATE \quad $z_s, z_c, z_g \leftarrow \text{CSR}_s(f_\theta(x)),\; \text{CSR}_c(f_\theta(x)),\; \text{CSR}_g(f_\theta(x))$
\STATE \textbf{Text encoding:}
\STATE \quad $e_s, e_c, e_g \leftarrow g_\psi(T_s),\; g_\psi(T_c),\; g_\psi(T_g)$ \hfill $\triangleright$ branch supervision (training only)
\STATE \quad $\tilde{t}_0, \tilde{t}_1 \leftarrow \text{TextRefine}(g_\psi(\mathcal{T}_{\text{norm}}), g_\psi(\mathcal{T}_{\text{ano}}))$ \hfill $\triangleright$ normal/anomalous anchors
\STATE \textbf{CRM fusion:}
\STATE \quad $q \leftarrow W_q\tilde{t}_1, \quad K \leftarrow W_k[z_s, z_c, z_g]$
\STATE \quad $\alpha \leftarrow \text{softmax}(qK^\top / \sqrt{d})$
\STATE \quad $z_{\text{crm}} \leftarrow \sum_{b} \alpha_b z_b$
\STATE \textbf{Anomaly score:}
\STATE \quad $s(x) \leftarrow \sigma\!\left(\text{sim}(z_{\text{crm}}, \tilde{t}_1) - \text{sim}(z_{\text{crm}}, \tilde{t}_0)\right)$
\STATE \textbf{Training losses:}
\STATE \quad $\mathcal{L}_{\text{align}} \leftarrow \sum_{b}(1 - \text{sim}(z_b, e_b))$
\STATE \quad $\mathcal{L}_{\text{decorr}} \leftarrow \sum_{i \neq j} |\text{sim}(z_i, z_j)|^2$
\STATE \quad $\mathcal{L}_{\text{fuse}} \leftarrow \text{CE}([\text{sim}(z_{\text{crm}},\tilde{t}_0),\, \text{sim}(z_{\text{crm}},\tilde{t}_1)],\, y) + \lambda_{\text{fc}}\mathcal{L}_{\text{fuse\text{-}cons}} -\lambda_{\text{fe}}\mathcal{L}_{\text{fuse\text{-}ent}}$
\STATE \quad $\mathcal{L}_{\text{text}} \leftarrow \lambda_o \mathcal{L}_{\text{ortho}} + \lambda_c \mathcal{L}_{\text{cons}} + \lambda_g \mathcal{L}_{\text{ground}}$
\STATE \quad $\mathcal{L} \leftarrow \mathcal{L}_{\text{align}} + \mathcal{L}_{\text{decorr}} + \mathcal{L}_{\text{fuse}} + \mathcal{L}_{\text{text}}$
\end{algorithmic}
\end{algorithm}

\subsection{Text-Conditioned Branch Specialization}

CC-CLIP produces three complementary representations from the same image by equipping a shared frozen CLIP backbone with branch-specific adapters guided by distinct textual roles: subject identity, surrounding context, and global scene.

\textbf{Context-Selective Residuals (CSR).} For each branch $b \in \{s, c, g\}$, a lightweight residual adapter is inserted into the first $K$ transformer layers. At layer $i$, token representations are updated as:
\begin{equation}
\label{eq:csr_module}
    \mathbf{x}^{(i)}_b \leftarrow (1 - \lambda_i)\, \mathbf{x}^{(i)} + \lambda_i\, \text{CSR}_b^{(i)}\!\left(\mathbf{x}^{(i)}\right), \quad i = 1, \ldots, K,
\end{equation}
where $\lambda_i \in [0,1]$ is a learnable mixing coefficient and $\text{CSR}_b^{(i)}$ is a two-layer residual module. Layers beyond $K$ remain frozen, preserving CLIP's pretrained alignment. Each branch produces a class-token embedding $z_b \in \mathbb{R}^D$. The corresponding branch text embedding $e_b = g_\psi(T_b)$ is obtained by encoding the role-specific prompt $T_b$ through the frozen CLIP text encoder $g_\psi$.

\textbf{Text-guided specialization.} Each branch is supervised to align with a role-specific text prompt (Algorithm~\ref{alg:cococlip}, lines 1--4), guiding each adapter to capture a distinct aspect of the image; these prompts are used only during training and not at inference:
\begin{equation}
    \mathcal{L}_{\text{align}} = \sum_{b \in \{s,c,g\}} \left(1 - \text{sim}(z_b, e_b)\right).
\end{equation}

\textbf{Branch decorrelation.} To prevent representational collapse, we enforce decorrelation across branches:
\begin{equation}
    \mathcal{L}_{\text{decorr}} = \sum_{i \neq j} \left|\text{sim}(z_i, z_j)\right|^2
\end{equation}
ensuring branches capture complementary information as required by Proposition~\ref{prop:nonidentifiability}. Alignment defines \emph{what} each branch learns; decorrelation prevents their representations from collapsing toward one another.

\subsection{Text Refinement for Contextual States}
\label{text_refinement}
Contextual anomaly detection requires the text space to represent two distinct semantic 
states of the same concept: contextually normal and contextually anomalous. We refine 
CLIP's text representations using a lightweight adaptation of the first $L$ transformer 
layers via a gated residual pathway:
\begin{equation}
    \mathbf{z}^{(i)} \leftarrow (1 - \gamma_i)\, \mathbf{z}^{(i)} + \gamma_i\, \text{TR}^{(i)}\!\left(\mathbf{z}^{(i)}\right), \quad i = 1, \ldots, L,
\end{equation}
where $\gamma_i$ is a learnable gate and $\text{TR}^{(i)}$ is a two-layer MLP. This 
produces paired embeddings $(\tilde{t}_0, \tilde{t}_1)$ representing normal and 
anomalous contextual interpretations. The paired embeddings are optimized using three 
complementary objectives:
\begin{equation}
\label{eq:text_refinement}
    \mathcal{L}_{\text{text}} = \lambda_o \underbrace{\langle \tilde{t}_0, \tilde{t}_1 \rangle^2}_{\mathcal{L}_{\text{ortho}}} + \lambda_c \underbrace{\sum_{k \in \{0,1\}} \|\tilde{t}_k - \tilde{t}_\mu\|^2_2}_{\mathcal{L}_{\text{cons}}} + \lambda_g \underbrace{\left(1 - \text{sim}(\tilde{t}_\mu, v_{\text{cls}})\right)}_{\mathcal{L}_{\text{ground}}},
\end{equation}
where $\tilde{t}_\mu = \frac{1}{2}(\tilde{t}_0 + \tilde{t}_1)$ is the un-normalized 
mean and $v_{\text{cls}}$ is the frozen CLIP image encoder's class token. Orthogonality ($\mathcal{L}_{\text{ortho}}$) 
separates normal from anomalous embeddings, consistency ($\mathcal{L}_{\text{cons}}$) 
preserves shared class identity in the full embedding space rather than scalar cosine 
similarity alone, and grounding ($\mathcal{L}_{\text{ground}}$) maintains alignment 
with the visual domain.

\subsection{Compatibility Reasoning and Anomaly Scoring}
\label{compatrm}

The Compatibility Reasoning Module (CRM) aggregates the three representations via text-conditioned attention. Using the anomalous text anchor $\tilde{t}_1$ as a query, CRM computes:
\begin{equation}
\label{eq:CRM_1}
    q = W_q\tilde{t}_1, \quad K = W_k[z_s, z_c, z_g], \quad 
\alpha = \text{softmax}\!\left(\frac{qK^\top}{\sqrt{d}}\right)
\end{equation}
where $W_q, W_k \in \mathbb{R}^{d \times D}$ are learned projections. The fused representation is:
\begin{equation}
\label{eq:CRM_2}
    z_{\text{crm}} = \sum_{b \in \{s,c,g\}} \alpha_b\, z_b.
\end{equation}

The image-level anomaly score is computed as the cosine margin between anomalous and normal compatibility:
\begin{equation}
\label{eq:CRM_3}
    s(x) = \sigma\!\left(\text{sim}(z_{\text{crm}}, \tilde{t}_1) - \text{sim}(z_{\text{crm}}, \tilde{t}_0)\right),
\end{equation}
where $\sigma$ is the sigmoid function. Higher scores indicate stronger subject--context incompatibility.

CRM is supervised via cross-entropy on the fused representation, with consistency and entropy regularization:
\begin{equation}
\label{eq:CRM_4}
    \mathcal{L}_{\text{fuse}} = \text{CE}\!\left(\text{sim}(z_{\text{crm}}, \tilde{t}_0),\, \text{sim}(z_{\text{crm}}, \tilde{t}_1),\, y\right) + \lambda_{\text{fc}}\left\|\ell_{\text{crm}} - \tfrac{1}{3}\textstyle\sum_b \ell_b\right\|^2_2 - \lambda_{\text{fe}}\, H(\alpha),
\end{equation}
where $\ell_{\text{crm}} = [\text{sim}(z_{\text{crm}},\tilde{t}_0), \text{sim}(z_{\text{crm}},\tilde{t}_1)]$ and $\ell_b = [\text{sim}(z_b,\tilde{t}_0), \text{sim}(z_b,\tilde{t}_1)]$ are the scaled fused and branch logit vectors, and $H(\alpha)$ is the entropy of the attention weights.

\textbf{Training objective.} CC-CLIP is trained end-to-end using:
\begin{equation}
    \mathcal{L} = \mathcal{L}_{\text{align}} + \mathcal{L}_{\text{decorr}} + \mathcal{L}_{\text{fuse}} + \mathcal{L}_{\text{text}}
\end{equation}
The CLIP backbone remains frozen throughout; only CSR adapters, text refinement layers, and CRM projections are trained.

\section{Experiments}
\subsection{Experimental Setup}
\label{sec:setup}

\textbf{Datasets.}
No existing benchmark supports training and evaluating contextual anomaly detection models: MIT-OOC~\citep{choi2012context} contains only 42 test images with no training split, and COCO-OOC~\citep{acharya2022detecting} targets physical implausibility rather than semantic incompatibility. To address this, we construct \textbf{CAAD}, a benchmark for contextual anomaly detection comprising two complementary splits. \textbf{CAAD-Syn} contains 23K synthetic images spanning 15 object and action categories with diverse scene contexts and balanced normal/anomalous labels; synthetic image quality is validated against real-image benchmarks in Appendix~\ref{app:dataset}. \textbf{CAAD-Real} contains 1,026 real-world images (547 normal, 479 anomalous) collected from diverse web sources, covering the same categories. All models are trained exclusively on CAAD-Syn and evaluated on CAAD-Real, ensuring that reported results reflect genuine synthetic-to-real transfer. 

Following the contextual anomaly detection setting, training uses both normal and anomalous examples per class, as discriminating subject--context compatibility requires exposure to both compatible and incompatible pairings. Normal-only training results are provided in Appendix~\ref{app:ablations}.


\textbf{Baselines.}
We compare against three categories of methods. (i)~\emph{CLIP-based anomaly detectors}: 
WinCLIP~\citep{jeong2023winclip}, AnomalyCLIP~\citep{zhou2023anomalyclip}, 
AdaCLIP~\citep{cao2024adaclip}, PromptAD~\citep{li2024promptad}, 
AA-CLIP~\citep{ma2025aaclip}, and IIPAD~\citep{lv2025oneforall}. 
(ii)~\emph{Context-reasoning methods}: CRTNet~\citep{bomatter2021pigs}. 
(iii)~\emph{Vision--language models}: InternVL2~\citep{chen2024internvl} and 
Qwen2-VL~\citep{wang2024qwen2vl}, evaluated zero-shot without fine-tuning to assess 
out-of-the-box VLM capability relative to task-specific models. Prompting details 
and extended variants are in Appendix~\ref{app:vlm_prompts}. All CLIP-based baselines 
are retrained under identical settings using the same backbone and data splits.

\textbf{Metrics.}
We report image-level AUROC (I-AUROC) and average precision (I-AP) on CAAD-Real, focusing on image-level anomaly detection.

\textbf{Implementation.}
CC-CLIP uses a frozen CLIP ViT-L/14 backbone. Only CSR adapters ($<$7\% additional parameters), text refinement layers, and CRM projections are trained. Training uses Adam with cosine decay for 20 epochs on a single H100 GPU. Full hyperparameters are in Appendix~\ref{app:hyperparams}.

All results are averaged over three independent runs with different random seeds.

\subsection{Main Results}
\label{sec:results}


\begin{table*}[t]
\centering
\caption{Contextual anomaly detection on \textbf{CAAD-Real}. All methods train on CAAD-Syn and evaluate on real-world images. I-AUROC and I-AP reported. \textbf{Bold} and \underline{underline} denote best and second best; dashes indicate not applicable.}
\label{tab:main}
\resizebox{\textwidth}{!}{%
\begin{tabular}{llcccccc}
\toprule
& & \multicolumn{2}{c}{\textbf{4-shot}} & \multicolumn{2}{c}{\textbf{16-shot}} & \multicolumn{2}{c}{\textbf{Full}} \\
\cmidrule(lr){3-4} \cmidrule(lr){5-6} \cmidrule(lr){7-8}
\textbf{Method} & \textbf{Type} & I-AUROC & I-AP & I-AUROC & I-AP & I-AUROC & I-AP \\
\midrule
InternVL-1B~\citep{chen2024internvl} & VLM & -- & -- & -- & -- & 44.53 & 50.20 \\
Qwen2-VL-2B~\citep{wang2024qwen2vl} & VLM & -- & -- & -- & -- & \underline{73.12} & \underline{71.20} \\
\midrule
CRTNet~\citep{bomatter2021pigs} & OOC & 52.23 & 54.85 & 58.56 & 60.23 & 61.42 & 64.30 \\
\midrule
CLIP~\citep{radford2021learning} & CLIP & 36.20 & 40.52 & 38.56 & 42.58 & 41.20 & 45.20 \\
WinCLIP~\citep{jeong2023winclip} & CLIP & 48.05 & 52.52 & 51.23 & 54.56 & 55.24 & 57.42 \\
AnomalyCLIP~\citep{zhou2023anomalyclip} & CLIP & \underline{58.60} & 57.25 & 59.10 & 60.95 & 60.10 & 61.75 \\
AdaCLIP~\citep{cao2024adaclip} & CLIP & 55.62 & 56.23 & 52.89 & 55.78 & 54.26 & 56.28 \\
PromptAD~\citep{li2024promptad} & CLIP & 55.20 & 59.60 & 59.23 & \underline{62.10} & 59.56 & 60.23 \\
AA-CLIP~\citep{ma2025aaclip} & CLIP & 57.58 & \underline{60.20} & \underline{60.55} & 61.87 & 61.52 & 65.61 \\
IIPAD~\citep{lv2025oneforall} & CLIP & 42.69 & 43.63 & 45.20 & 45.56 & 45.03 & 47.59 \\
\midrule
\textbf{CC-CLIP (Ours)} & -- & \textbf{74.23} & \textbf{74.55} & \textbf{77.51} & \textbf{78.84} & \textbf{80.21} & \textbf{81.12} \\
\bottomrule
\end{tabular}%
}
\end{table*}

Table~\ref{tab:main} reports contextual anomaly detection on CAAD-Real, where all 
methods train on CAAD-Syn and evaluate without fine-tuning. CC-CLIP substantially 
outperforms all baselines across every shot setting. CLIP-based detectors 
consistently fail on contextual anomalies despite strong structural benchmark performance, 
as their global representations entangle subject and context per 
Proposition~\ref{prop:nonidentifiability}. CRTNet improves over CLIP-based approaches 
but remains limited by appearance-level cues. VLM baselines (InternVL, Qwen2-VL), 
despite powerful reasoning capabilities and potential exposure to web-sourced CAAD-Real 
images during pretraining, lack the fine-grained subject--context decomposition that 
CC-CLIP learns through text-conditioned specialization. Notably, CC-CLIP with 4 
shots surpasses all baselines at full supervision, demonstrating strong data efficiency.


\subsection{Generalization to OOC Benchmarks}
\label{sec:ooc}


We evaluate cross-dataset generalization by testing CC-CLIP on MIT-OOC~\citep{choi2012context} and COCO-OOC~\citep{acharya2022detecting}, trained on CAAD-Syn and applied zero-shot without dataset-specific adaptation with threshold of $0.5$. As shown in Table~\ref{tab:ooc}, CC-CLIP achieves the best performance on both benchmarks, outperforming prior methods by a clear margin, indicating strong transferability and robust subject--context reasoning.

\begin{table}[!ht]
\centering
\caption{Out-of-context detection accuracy (\%) on established OOC benchmarks. CC-CLIP is trained on CAAD-Syn and evaluated zero-shot. Prior results are from their respective publications.}
\label{tab:ooc}
\begin{tabular}{lcc}
\toprule
\textbf{Method} & \textbf{MIT-OOC} & \textbf{COCO-OOC} \\
\midrule
COC~\citep{choi2012context} & 73.29 & -- \\
GCRN~\citep{acharya2022detecting} & -- & 84.85 \\
FM Zero-shot~\citep{roy2025zeroshot} & 90.82 & 87.26 \\
\midrule
\textbf{CC-CLIP (Ours)} & \textbf{95.24} & \textbf{97.20} \\
\bottomrule
\end{tabular}
\end{table}


\subsection{Ablation Studies}
\label{sec:ablation}

We ablate the key components of CC-CLIP on CAAD-Real under the full training setting. Extended ablations are provided in Appendix~\ref{app:ablations}.

\textbf{Component analysis.} Table~\ref{tab:component} shows the incremental contribution of each module. Frozen CLIP fails entirely on contextual anomalies, confirming the non-identifiability predicted by Proposition~\ref{prop:nonidentifiability}. CSR provides the largest single gain by enabling branch-level specialization. Text refinement further improves performance by sharpening the separation between normal and anomalous semantic anchors, and CRM contributes additional gains through text-conditioned fusion.

\begin{table}[!ht]
\centering
\begin{minipage}[t]{0.48\textwidth}
\centering
\caption{Component ablation on CAAD-Real.}
\label{tab:component}
\begin{tabular}{lc}
\toprule
\textbf{Configuration} & \textbf{I-AUROC} \\
\midrule
Frozen CLIP (no adaptation) & 41.20 \\
+ CSR & 65.53 \\
+ CSR + Text refinement & 72.36 \\
+ CSR + Text + CRM (Full) & \textbf{80.21} \\
\bottomrule
\end{tabular}
\end{minipage}
\hfill
\begin{minipage}[t]{0.48\textwidth}
\centering
\caption{Design ablation on CAAD-Real.}
\label{tab:design}
\begin{tabular}{lc}
\toprule
\textbf{Configuration} & \textbf{I-AUROC} \\
\midrule
Full CC-CLIP & \textbf{80.21} \\
w/o $\mathcal{L}_{\text{decorr}}$ & 72.80 \\
w/o text specialization & 66.15 \\
Global branch only & 60.45 \\
Subject branch only & 63.22 \\
Context branch only & 58.87 \\
\bottomrule
\end{tabular}
\end{minipage}
\end{table}



\textbf{Design analysis.} Table~\ref{tab:design} validates three critical design decisions. Text-guided specialization is the largest contributor; replacing branch-specific prompts with identical ones causes the steepest drop, confirming language-driven decomposition as the primary mechanism. Removing decorrelation also degrades performance substantially, showing explicit anti-collapse supervision is necessary even with distinct text guidance. Finally, no single branch achieves competitive performance and collapsed branches without $\mathcal{L}_{\text{decorr}}$ still outperform them via residual diversity, jointly validating that decomposed representations are necessary for contextual reasoning, as established by Proposition~\ref{prop:nonidentifiability}.

\begin{table}[!ht]
\centering
\begin{minipage}[t]{0.48\textwidth}
\centering
\caption{Text disentanglement loss ablation.}
\label{tab:textloss}
\begin{tabular}{lc}
\toprule
\textbf{Configuration} & \textbf{I-AUROC} \\
\midrule
Full CC-CLIP & \textbf{80.21} \\
w/o $\mathcal{L}_{\text{ortho}}$ & 74.50 \\
w/o $\mathcal{L}_{\text{cons}}$ & 76.20 \\
w/o $\mathcal{L}_{\text{ground}}$ & 73.80 \\
w/o $\mathcal{L}_{\text{ortho}}$ + $\mathcal{L}_{\text{cons}}$ & 71.10 \\
Only CE (no text losses) & 70.10 \\
\bottomrule
\end{tabular}
\end{minipage}
\hfill
\begin{minipage}[t]{0.48\textwidth}
\centering
\caption{Fusion mechanism comparison.}
\label{tab:fusion}
\begin{tabular}{lc}
\toprule
\textbf{Fusion Method} & \textbf{I-AUROC} \\
\midrule
Average fusion & 70.30 \\
Static learned weights & 71.50 \\
Concat + Linear & 72.20 \\
CRM (Ours) & \textbf{80.21} \\
\bottomrule
\end{tabular}
\end{minipage}
\end{table}


\textbf{Text loss analysis.} Table~\ref{tab:textloss} isolates each text disentanglement term. All three losses contribute complementary benefits, with grounding and orthogonality having the largest impact. Removing both together causes a larger drop than either alone, confirming complementarity: orthogonality separates normal from anomalous anchors while consistency preserves shared class identity.

\textbf{Fusion analysis.} Table~\ref{tab:fusion} compares CRM against standard fusion strategies using the same tri-branch features. Simple averaging, static weights, and linear projection all underperform CRM. This confirms that text-conditioned adaptive routing is essential for contextual anomaly detection, as static fusion cannot adjust branch contributions based on the subject--context relationship in each image.

\begin{table}[!ht]
\centering
\begin{minipage}[t]{0.48\textwidth}
\centering
\caption{Parameter and inference comparison.}
\label{tab:efficiency}
\small
\setlength{\tabcolsep}{3pt}
\begin{tabular}{lcccc}
\toprule
\textbf{Method} & \textbf{Trainable} & \textbf{Total} & \textbf{ms/img} \\
\midrule
CRTNet & 32M & 105M & 26 \\
AdaCLIP & 5.56M & 427M & 25 \\
AA-CLIP & 12.58M & 434M & 19 \\
IIPAD & 1.5B & 1.9B & 21 \\
\midrule
CC-CLIP & 34.80M & 462M & 18 \\
\bottomrule
\end{tabular}
\end{minipage}
\hfill
\begin{minipage}[t]{0.48\textwidth}
\centering
\vspace{0pt}
\includegraphics[width=\linewidth]{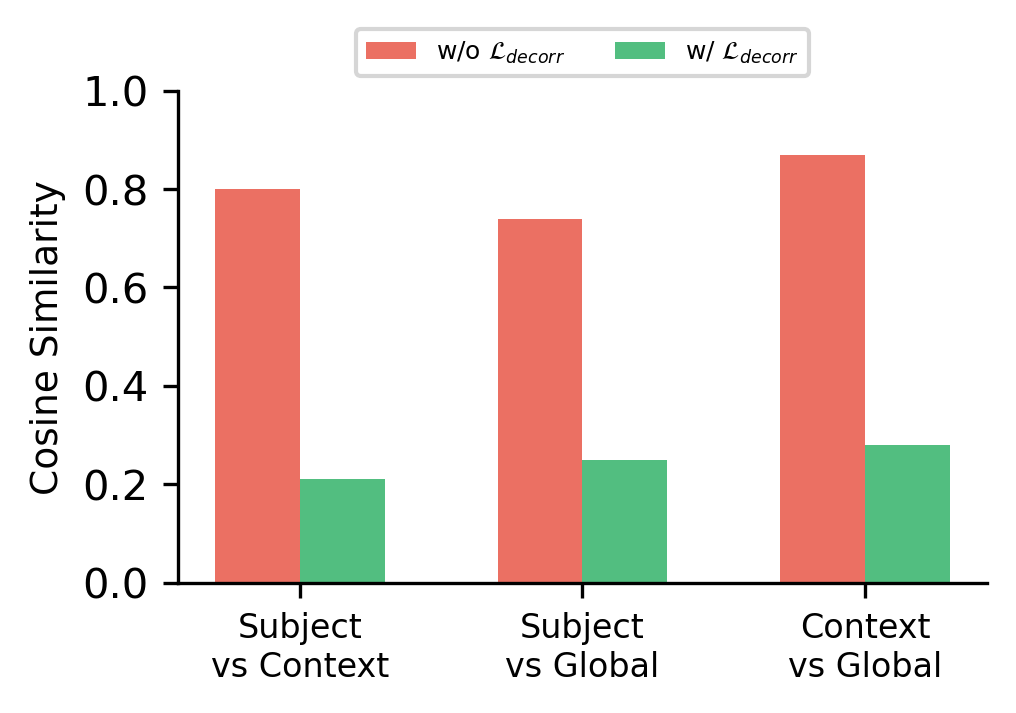}
\captionof{figure}{Branch cosine similarity with and without $\mathcal{L}_{\text{decorr}}$.}
\label{fig:decorr}
\end{minipage}
\end{table}

\textbf{Efficiency.} Table~\ref{tab:efficiency} compares trainable parameters and inference cost. CC-CLIP trains more parameters than AdaCLIP or AA-CLIP, but this overhead corresponds directly to the tri-branch adapters and CRM fusion that drive the performance gains; each component is a lightweight adapter on a frozen backbone 
requiring no iterative optimization or dataset-specific tuning. Inference speed remains competitive at 18ms per image despite parallel branch execution. 


\textbf{Branch decorrelation.} Figure~\ref{fig:decorr} shows that without $\mathcal{L}_{\text{decorr}}$, all three branches collapse to near-identical features despite distinct text supervision, confirming that text guidance alone is insufficient. With decorrelation, pairwise cosine similarity drops substantially, validating the design choice and explaining the performance gap in Table~\ref{tab:design}.

\subsection{Qualitative Analysis}
\label{sec:qualitative}

Figure~\ref{fig:qualitative} visualizes CC-CLIP's anomaly attribution on CAAD-Real. The model consistently localizes the subject responsible for the contextual violation, confirming that the learned representations capture subject--context incompatibility rather than appearance-level cues. Additional visualizations across all categories are provided in Appendix~\ref{app:qualitative}. Notably, CC-CLIP correctly handles multi-object scenes (e.g., the herd of cows in Figure~\ref{fig:qualitative}), localizing all contextually incompatible entities without requiring per-instance supervision.

\begin{figure}[t]
\centering
\setlength{\tabcolsep}{1pt}
\begin{tabular}{cccc}
\includegraphics[width=0.24\linewidth]{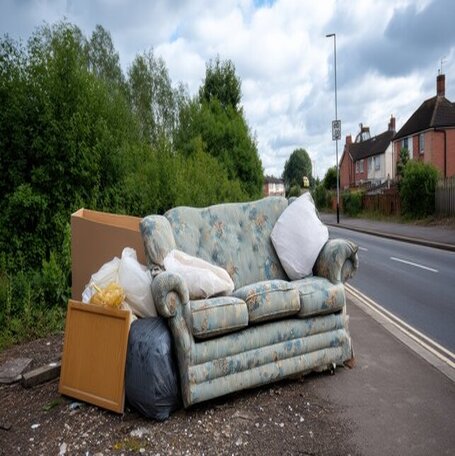} &
\includegraphics[width=0.24\linewidth]{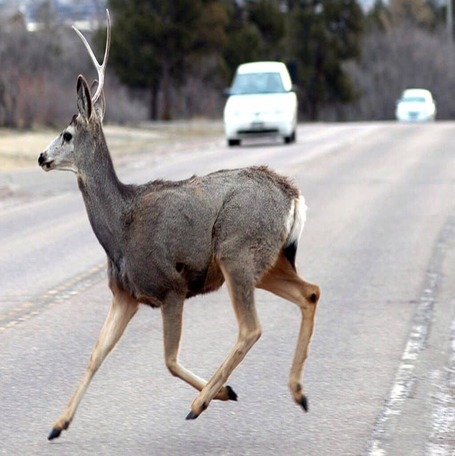} &
\includegraphics[width=0.24\linewidth]{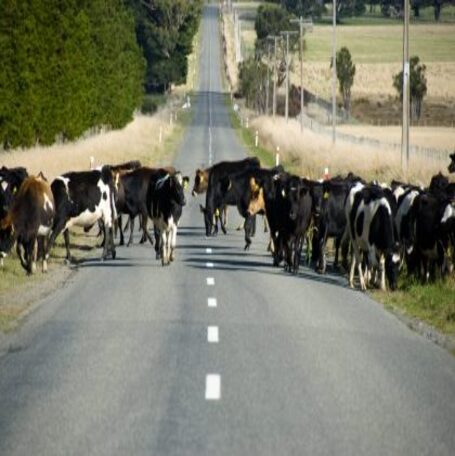} &
\includegraphics[width=0.24\linewidth]{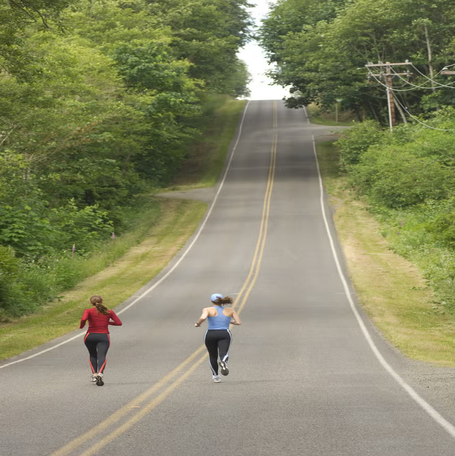} \\
\includegraphics[width=0.24\linewidth]{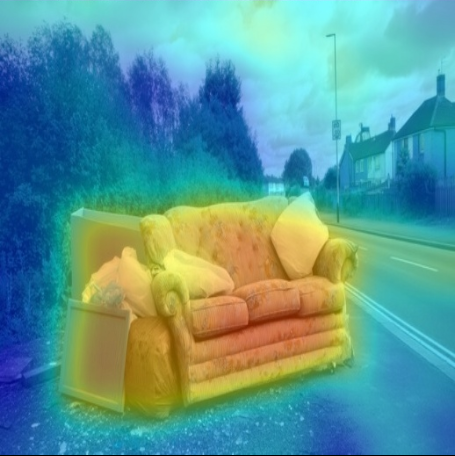} &
\includegraphics[width=0.24\linewidth]{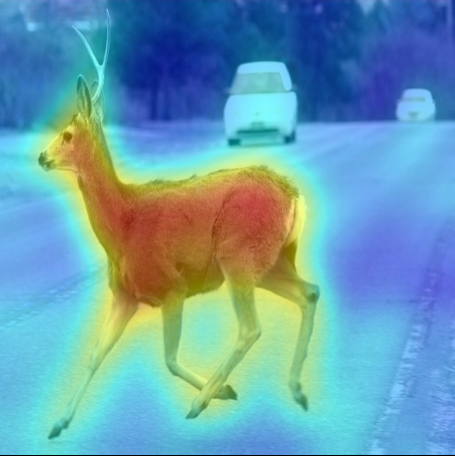} &
\includegraphics[width=0.24\linewidth]{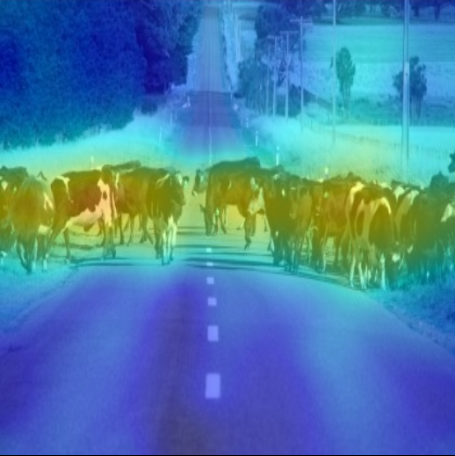} &
\includegraphics[width=0.24\linewidth]{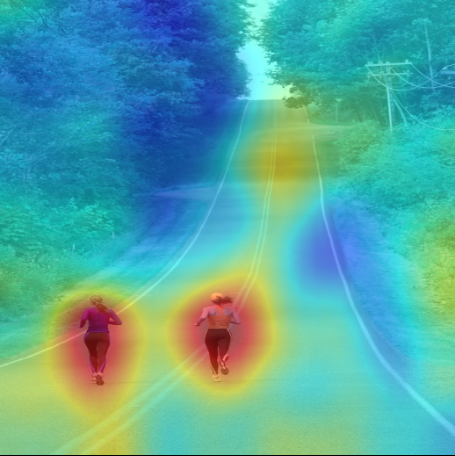} \\
\small\textit{Sofa} & \small\textit{Animal} & \small\textit{Animal} & \small\textit{Person running} \\
\end{tabular}
\caption{Anomaly heatmaps on CAAD-Real. CC-CLIP localizes the contextually incompatible subject in each scene.}
\label{fig:qualitative}
\end{figure}

\section{Conclusion}
\label{sec:conclusion}

We addressed the problem of contextual anomaly detection, where the same subject can be normal or anomalous depending solely on its surrounding context. We formalized a non-identifiability result showing that standard global representations are provably insufficient for this setting and proposed conditional compatibility learning as a principled alternative. Our instantiation, CC-CLIP, achieves text-conditioned decomposition of a single image into complementary representations without any spatial supervision, enabling subject--context compatibility reasoning through language-guided specialization and adaptive fusion. Trained entirely on synthetic data, CC-CLIP substantially outperforms all existing methods on real-world contextual anomaly detection and generalizes effectively to established OOC benchmarks. We believe contextual anomaly detection is a critical yet underexplored problem in visual understanding, and we release our benchmark and code to support future research in this direction.

\textbf{Limitations.} Our formulation assumes a dominant subject entity whose compatibility with the context determines the anomaly label. Extending to scenes with multiple interacting subjects, where the anomaly arises from inter-object relationships rather than subject--context mismatch, remains open. 

{
\small
\bibliographystyle{plain}
\bibliography{mybibliography}

@String(IJCAI = {IJCAI})

@inproceedings{bergmann2019mvtec,
  title={{MVTec AD}--A comprehensive real-world dataset for unsupervised anomaly detection},
  author={Bergmann, Paul and Fauser, Michael and Sattlegger, David and Steger, Carsten},
  booktitle={Proceedings of the IEEE/CVF conference on computer vision and pattern recognition},
  pages={9592--9600},
  year={2019}
}

@inproceedings{zou2022spot,
  title={Spot-the-difference self-supervised pre-training for anomaly detection and segmentation},
  author={Zou, Yang and Jeong, Jongheon and Pemula, Latha and Zhang, Dongqing and Dabeer, Onkar},
  booktitle={European Conference on Computer Vision},
  pages={392--408},
  year={2022},
  organization={Springer}
}

@inproceedings{sultani2018real,
  title={Real-world anomaly detection in surveillance videos},
  author={Sultani, Waqas and Chen, Chun and Shah, Mubarak},
  booktitle={Proceedings of the IEEE Conference on Computer Vision and Pattern Recognition},
  pages={6479--6488},
  year={2018}
}

@article{wang2024qwen2vl,
  title={{Qwen2-VL}: Enhancing Vision-Language Model's Perception of the World at Any Resolution},
  author={Wang, Peng and Bai, Shuai and Tan, Sinan and Wang, Shijie and Fan, Zhihao and Bai, Jinze and Chen, Keqin and Liu, Xuejing and Wang, Jialin and Ge, Wenbin and Fan, Yang and Dang, Kai and Du, Mengfei and Ren, Xuancheng and Men, Rui and Liu, Dayiheng and Zhou, Chang and Zhou, Jingren and Lin, Junyang},
  journal={arXiv preprint arXiv:2409.12191},
  year={2024}
}

@inproceedings{acsintoae2022ubnormal,
  title={UBnormal: New benchmark for supervised open-set video anomaly detection},
  author={Acsintoae, Andra and Florescu, Andrei and Georgescu, Mariana-Iuliana and Mare, Tudor and Sumedrea, Paul and Ionescu, Radu Tudor and Khan, Fahad Shahbaz and Shah, Mubarak},
  booktitle={Proceedings of the IEEE/CVF Conference on Computer Vision and Pattern Recognition},
  pages={20143--20153},
  year={2022}
}

@inproceedings{radford2021learning,
  title={Learning transferable visual models from natural language supervision},
  author={Radford, Alec and Kim, Jong Wook and Hallacy, Chris and Ramesh, Aditya and Goh, Gabriel and Agarwal, Sandhini and Sastry, Girish and Askell, Amanda and Mishkin, Pamela and Clark, Jack and others},
  booktitle={International Conference on Machine Learning},
  pages={8748--8763},
  year={2021},
  organization={PMLR}
}

@inproceedings{li2024promptad,
  title={{Promptad}: Zero-shot anomaly detection using text prompts},
  author={Li, Yiting and Goodge, Adam and Liu, Fayao and Foo, Chuan-Sheng},
  booktitle={Proceedings of the IEEE/CVF Winter Conference on Applications of Computer Vision},
  pages={1093--1102},
  year={2024}
}

@inproceedings{ma2025aaclip,
  title={{Aa-clip}: Enhancing zero-shot anomaly detection via anomaly-aware clip},
  author={Ma, Wenxin and Zhang, Xu and Yao, Qingsong and Tang, Fenghe and Wu, Chenxu and Li, Yingtai and Yan, Rui and Jiang, Zihang and Zhou, S Kevin},
  booktitle={Proceedings of the Computer Vision and Pattern Recognition Conference},
  pages={4744--4754},
  year={2025}
}

@inproceedings{jeong2023winclip,
  title={{Winclip}: Zero-/few-shot anomaly classification and segmentation},
  author={Jeong, Jongheon and Zou, Yang and Kim, Taewan and Zhang, Dongqing and Ravichandran, Avinash and Dabeer, Onkar},
  booktitle={Proceedings of the IEEE/CVF Conference on Computer Vision and Pattern Recognition},
  pages={19606--19616},
  year={2023}
}

@inproceedings{cao2024adaclip,
  title={{Adaclip}: Adapting clip with hybrid learnable prompts for zero-shot anomaly detection},
  author={Cao, Yunkang and Zhang, Jiangning and Frittoli, Luca and Cheng, Yuqi and Shen, Weiming and Boracchi, Giacomo},
  booktitle={European Conference on Computer Vision},
  pages={55--72},
  year={2024},
  organization={Springer}
}

@inproceedings{lv2025oneforall,
  title={One-for-All Few-Shot Anomaly Detection via Instance-Induced Prompt Learning},
  author={Lv, Wenxi and Su, Qinliang and Xu, Wenchao},
  booktitle={The Thirteenth International Conference on Learning Representations},
  year={2025}
}

@article{zhou2023anomalyclip,
  title={{Anomalyclip}: Object-agnostic prompt learning for zero-shot anomaly detection},
  author={Zhou, Qihang and Pang, Guansong and Tian, Yu and He, Shibo and Chen, Jiming},
  journal={arXiv preprint arXiv:2310.18961},
  year={2023}
}

@inproceedings{deng2022anomaly,
  title={Anomaly detection via reverse distillation from one-class embedding},
  author={Deng, Hanqiu and Li, Xingyu},
  booktitle={Proceedings of the IEEE/CVF conference on computer vision and pattern recognition},
  pages={9737--9746},
  year={2022}
}

@article{bergmann2022beyond,
  title={Beyond dents and scratches: Logical constraints in unsupervised anomaly detection and localization},
  author={Bergmann, Paul and Batzner, Kilian and Fauser, Michael and Sattlegger, David and Steger, Carsten},
  journal={International Journal of Computer Vision},
  volume={130},
  number={4},
  pages={947--969},
  year={2022},
  publisher={Springer}
}

@inproceedings{chen2024internvl,
  title={{Internvl}: Scaling up vision foundation models and aligning for generic visual-linguistic tasks},
  author={Chen, Zhe and Wu, Jiannan and Wang, Wenhai and Su, Weijie and Chen, Guo and Xing, Sen and Zhong, Muyan and Zhang, Qinglong and Zhu, Xizhou and Lu, Lewei and others},
  booktitle={Proceedings of the IEEE/CVF conference on computer vision and pattern recognition},
  pages={24185--24198},
  year={2024}
}

@misc{blackforestlabs2025flux2dev,
  title = {{FLUX}.2 [dev]},
  author = {{Black Forest Labs}},
  year = {2025},
  url = {https://huggingface.co/black-forest-labs/FLUX.2-dev}
}

@article{choi2012context,
  title={Context models and out-of-context objects},
  author={Choi, Myung Jin and Torralba, Antonio and Willsky, Alan S},
  journal={Pattern Recognition Letters},
  volume={33},
  number={7},
  pages={853--862},
  year={2012},
  publisher={Elsevier}
}

@inproceedings{bomatter2021pigs,
  title={{When pigs fly}: Contextual reasoning in synthetic and natural scenes},
  author={Bomatter, Philipp and Zhang, Mengmi and Karev, Dimitar and Madan, Spandan and Tseng, Claire and Kreiman, Gabriel},
  booktitle={Proceedings of the IEEE/CVF International Conference on Computer Vision},
  pages={255--264},
  year={2021}
}

@inproceedings{acharya2022detecting,
  title={Detecting out-of-context objects using graph context reasoning network},
  author={Acharya, Manoj and Roy, Anirban and Koneripalli, Kaushik and Jha, Susmit and Kanan, Christopher and Divakaran, Ajay},
  booktitle={IJCAI},
  year={2022}
}

@inproceedings{roy2025zeroshot,
  title={Zero-Shot Detection of Out-of-Context Objects Using Foundation Models},
  author={Roy, Anirban and Cobb, Adam and Kaur, Ramneet and Jha, Sumit and Bastian, Nathaniel D and Berenbeim, Alexander and Thomson, Robert and Cruickshank, Iain and Velasquez, Alvaro and Jha, Susmit},
  booktitle={2025 IEEE/CVF Winter Conference on Applications of Computer Vision (WACV)},
  pages={9186--9195},
  year={2025},
  organization={IEEE}
}

@article{chandola2009anomaly,
  title={Anomaly detection: A survey},
  author={Chandola, Varun and Banerjee, Arindam and Kumar, Vipin},
  journal={ACM computing surveys (CSUR)},
  volume={41},
  number={3},
  pages={1--58},
  year={2009},
  publisher={ACM New York, NY, USA}
}
}

\clearpage
\appendix

\section*{Appendix Overview}
\label{app:overview}

This appendix provides supplementary material organized as follows:

\begin{itemize}
    \item \textbf{Appendix~\ref{app:proof}}: Full proof of Proposition~\ref{prop:nonidentifiability}.
    \item \textbf{Appendix~\ref{app:model}}: Model architecture details including CSR, text refinement, CRM, and pixel-level anomaly mapping.
    \item \textbf{Appendix~\ref{app:hyperparams}}: Training protocol, hyperparameters, and loss weight sensitivity.
    \item \textbf{Appendix~\ref{app:dataset}}: Dataset construction, collection, statistics, and ethical considerations.
    \item \textbf{Appendix~\ref{app:ablations}}: Extended ablations including normal-only training, text loss joint sweep, corruption robustness and per-class results.
    \item \textbf{Appendix~\ref{app:structural}}: Zero-shot results on MVTec-AD and VisA structural anomaly benchmarks.
    \item \textbf{Appendix~\ref{app:impact}}: Broader impact and ethical considerations.
    \item \textbf{Appendix~\ref{app:qualitative}}: Additional qualitative results, and failure cases.
\end{itemize}

\section{Proof of Proposition~\ref{prop:nonidentifiability}}
\label{app:proof}

We restate the proposition and provide a complete proof with formal definitions.

\begin{definition}[Contextual Anomaly Detection Setting]
\label{def:setting}
Let $\mathcal{A}$ denote the space of subjects and $\mathcal{C}$ the space of contexts. A mapping $g: \mathcal{S} \times \mathcal{C} \to \mathcal{X}$, where $\mathcal{S} = \{a_1, \ldots, a_n\}$ and $a_i \in \mathcal{A}$, maps each subject set--context pair to an observed image $x \in \mathcal{X}$. A compatibility function $h: \mathcal{S} \times \mathcal{C} \to \{0, 1\}$ assigns the ground-truth anomaly label, where $h(\mathcal{S}, c) = 1$ indicates subject--context incompatibility. For clarity of exposition, the proof is stated for the single-subject case $\mathcal{S} = \{a\}$; the argument extends analogously to $|\mathcal{S}| > 1$.
\end{definition}

\begin{definition}[Intrinsic Representation and Detector]
\label{def:detector}
An \emph{intrinsic representation} is any measurable mapping $\phi: \mathcal{X} \to \mathcal{Z}$ that maps observations to a representation space $\mathcal{Z}$. An \emph{intrinsic detector} is a composite function $f: \mathcal{X} \to \{0, 1\}$ of the form $f(x) = \psi(\phi(x))$, where $\psi: \mathcal{Z} \to \{0, 1\}$ is an arbitrary decision function. We say $f$ is \emph{correct} on input $x = g(a, c)$ if $f(x) = h(a, c)$.
\end{definition}

\begin{propositionB}[Non-identifiability under intrinsic representations, restated]
Let $\phi: \mathcal{X} \to \mathcal{Z}$ be any intrinsic representation and $f(x) = \psi(\phi(x))$ any intrinsic detector. If there exists a subject $a \in \mathcal{A}$ and two contexts $c, c' \in \mathcal{C}$ such that
\begin{equation}
    \phi(g(a, c)) = \phi(g(a, c')) \quad \text{and} \quad h(a, c) \neq h(a, c'),
\end{equation}
then $f$ cannot be simultaneously correct on both $x = g(a, c)$ and $x' = g(a, c')$.
\end{propositionB}

\begin{proof}
We proceed by contradiction.

\textbf{Step 1: Setup.} Let $x = g(a, c)$ and $x' = g(a, c')$ be the two observations generated by the same subject $a$ in contexts $c$ and $c'$ respectively. By the premise:
\begin{align}
    \phi(x) &= \phi(g(a, c)) = \phi(g(a, c')) = \phi(x'), \label{eq:collapse} \\
    h(a, c) &\neq h(a, c'). \label{eq:labels}
\end{align}

\textbf{Step 2: Determinism of the detector.} The detector $f$ is a deterministic composition:
\begin{equation}
    f(x) = \psi(\phi(x)) \quad \text{and} \quad f(x') = \psi(\phi(x')).
\end{equation}
Since $\psi$ is a function and $\phi(x) = \phi(x')$ by \eqref{eq:collapse}, it follows that:
\begin{equation}
    f(x) = \psi(\phi(x)) = \psi(\phi(x')) = f(x'). \label{eq:equal_output}
\end{equation}

\textbf{Step 3: Correctness requirement.} For $f$ to be correct on both inputs simultaneously, we require:
\begin{align}
    f(x) &= h(a, c), \label{eq:correct1} \\
    f(x') &= h(a, c'). \label{eq:correct2}
\end{align}

\textbf{Step 4: Contradiction.} Substituting \eqref{eq:equal_output} into \eqref{eq:correct1} and \eqref{eq:correct2}:
\begin{equation}
    h(a, c) = f(x) = f(x') = h(a, c').
\end{equation}
This yields $h(a, c) = h(a, c')$, which directly contradicts \eqref{eq:labels}. Therefore, no intrinsic detector $f = \psi \circ \phi$ can be correct on both $x$ and $x'$. 
\end{proof}

\begin{remark}[Scope and implications]
\label{rem:scope}
Several observations follow from Proposition~\ref{prop:nonidentifiability}:

\begin{enumerate}
    \item \textbf{Generality.} The result holds for \emph{any} representation $\phi$ and \emph{any} decision function $\psi$, including deep neural networks with arbitrary capacity. The failure is not due to limited model expressiveness but due to information loss: an intrinsic representation that does not explicitly condition on context $c$ tends to suppress relational information when trained with object-centric objectives, discarding the contextual signal that determines the label. CLIP's contrastive pretraining objective optimizes for subject-level semantic alignment across contexts, which systematically suppresses contextual cues in favor of object-centric features, precisely the condition under which representation collisions identified in Proposition~\ref{prop:nonidentifiability} arise.

    \item \textbf{Practical relevance.} The condition $\phi(g(a,c)) = \phi(g(a,c'))$ is not pathological. Standard vision encoders, including CLIP, are pretrained with objectives that prioritize subject-level semantics over scene-level context. This induces a form of shortcut learning: the subject provides a high-gradient signal during pretraining, causing the encoder to suppress contextual cues. The result is a representation that is nearly invariant to context changes for a fixed subject, precisely the collision condition that Proposition~\ref{prop:nonidentifiability} identifies as the source of non-identifiability. We verify this empirically in Figure~\ref{fig:teaser}, where CLIP embeddings of the same subject in normal and anomalous contexts are nearly indistinguishable.

    \item \textbf{Necessary condition for resolution.} The proposition implies that any correct contextual anomaly detector must use a representation that does \emph{not} collapse $g(a,c)$ and $g(a,c')$ when $h(a,c) \neq h(a,c')$. This necessitates representations that separately encode subject and context information, which is the design principle underlying CC-CLIP.

    \item \textbf{Relation to non-identifiability in causal inference.} The result is analogous to non-identifiability under unobserved confounders in causal inference. The context $c$ acts as a necessary covariate for the label $y = h(a, c)$: observing only $x = g(a, c)$ through a representation $\phi$ that does not disentangle $c$ is analogous to marginalizing over a confounder; not because $c$ is unobserved in the image, but because the encoder's inductive bias suppresses it, effectively discarding the contextual signal needed for correct prediction. Explicit conditioning on context, as in our conditional compatibility formulation, resolves this identifiability failure in the same way that adjusting for confounders resolves identifiability in causal models.

    \item \textbf{Scope of the mapping function.} The mapping $g$ need not be invertible; Proposition~\ref{prop:nonidentifiability} requires only that two distinct subject--context pairs produce indistinguishable representations, which holds regardless of occlusion or ambiguity.

\end{enumerate}
\end{remark}

\section{Model Details}
\label{app:model}

This section provides architectural specifications for the components introduced in Section~\ref{sec:method}.

\subsection{CSR Architecture}
\label{app:csr}

Each Context-Selective Residual module $\text{CSR}_b^{(i)}$ (Eq.\ref{eq:csr_module}) consists of a two-layer MLP with a hidden dimension of $D/4$, where $D$ is the CLIP embedding dimension. Specifically:
\begin{equation}
    \text{CSR}_b^{(i)}(\mathbf{x}) = W_2^{(b,i)} \, \sigma\!\left(W_1^{(b,i)} \, \mathbf{x}\right),
\end{equation}
where $W_1^{(b,i)} \in \mathbb{R}^{(D/4) \times D}$, $W_2^{(b,i)} \in \mathbb{R}^{D \times (D/4)}$, and $\sigma$ is the GELU activation. The learnable mixing coefficient $\lambda_i$ is initialized to $0.01$ to ensure minimal perturbation of pretrained features at the start of training.

CSR adapters are inserted into the first $K{=}6$ transformer layers. Layers beyond $K$ remain frozen, preserving CLIP's deep-layer semantic alignment. Each branch $b \in \{s, c, g\}$ has its own set of CSR parameters, yielding $3 \times K = 18$ adapter modules in total. From the final transformer layer, we extract the class-token embedding $z_b \in \mathbb{R}^D$ (used for image-level scoring) and patch-token embeddings $P_b \in \mathbb{R}^{(N-1) \times D}$ (used for pixel-level anomaly mapping in Appendix~\ref{app:pixel}), where the representation reflects the cumulative effect of CSR adaptations applied across the first $K$ shared backbone layers.

\subsection{Text Refinement and Prompt Templates}
\label{app:text}

The text refinement module (Sec. \ref{text_refinement}) adapts the first $L{=}3$ layers of the CLIP text encoder. Each gated residual $\text{TR}^{(i)}$ is a two-layer MLP with hidden dimension $D/4$, mirroring the CSR design. The learnable gate $\gamma_i$ is initialized to $0.01$. The disentanglement objectives ($\mathcal{L}_{\text{ortho}}$, 
$\mathcal{L}_{\text{cons}}$, $\mathcal{L}_{\text{ground}}$) are fully defined in Eq.\ref{eq:text_refinement} of the main paper. Note that since $\tilde{t}_\mu$ in Eq.\ref{eq:text_refinement} is the un-normalized mean of $\tilde{t}_0$ and $\tilde{t}_1$, $\mathcal{L}_{\text{cons}}$ does not reduce to a function of cosine similarity alone, preserving the mathematical distinctness of $\mathcal{L}_{\text{ortho}}$ and $\mathcal{L}_{\text{cons}}$ as complementary objectives.

\textbf{Branch-specific prompts.} As described in Algorithm~\ref{alg:cococlip}, each branch receives a role-specific text prompt during training. For a class label \texttt{cls}, the prompts are:
\begin{itemize}
    \item Subject: \texttt{"a photo of [cls]"}
    \item Context: \texttt{"a scene with [cls] surroundings"}
    \item Global: \texttt{"a photo of [cls] in a [context]"}
\end{itemize}
The \texttt{[context]} placeholder in the global branch prompt is instantiated using scene descriptors from CAAD-Syn prompt annotations (e.g., ``highway'', ``beach'', ``living room'') during training. At inference, \texttt{[context]} is set to a generic placeholder \texttt{"scene"}, requiring no explicit context annotation. The branch embeddings $e_s, e_c, e_g$ are obtained by averaging over multiple prompt instantiations 
per role to reduce prompt sensitivity. At inference, branch prompts are not required; the CSR adapters are already specialized to decompose subject, context, and global representations directly from the input image.

\textbf{Normal/anomalous anchors.} The refined paired embeddings $(\tilde{t}_0, 
\tilde{t}_1)$ are produced from separate prompt sets:
\begin{itemize}
    \item Normal: \texttt{"a photo of [cls] in a normal place"}, \texttt{"a [cls] in a typical environment"}, \texttt{"a [cls] in a safe context"}
    \item Anomalous: \texttt{"a [cls] in an unusual place"}, \texttt{"a [cls] in an abnormal environment"}, \texttt{"a [cls] in an unsafe context"}
\end{itemize}
Final embeddings are obtained by averaging across templates and normalizing to unit length. These prompts require only the class name \texttt{cls} at inference, consistent with the prompt-based inference protocol of existing CLIP-based anomaly detection methods~\citep{jeong2023winclip, zhou2023anomalyclip, ma2025aaclip}.

\subsection{CRM Projections}
\label{app:crm}

The Compatibility Reasoning Module (Sec. \ref{compatrm}) uses learned linear projections $W_q, W_k \in \mathbb{R}^{d \times D}$ where $d = D/4$ is the attention dimension. The query is formed from the anomalous text anchor $\tilde{t}_1$, and keys are formed from the three branch embeddings $z_s, z_c, z_g$. The fused representation $z_{\text{crm}}$ is computed in the original embedding space $\mathbb{R}^D$ via weighted combination of the branch embeddings using the attention weights $\alpha$ (Eq. \ref{eq:CRM_2}).

To prevent single-branch dominance, the attention entropy $H(\alpha)$ is regularized (Eq. \ref{eq:CRM_4}). The fusion consistency term aligns fused logits with the mean of branch logits, ensuring that CRM reasoning is grounded in individual branch evidence rather than learning an independent prediction.

\subsection{Pixel-Level Anomaly Mapping}
\label{app:pixel}

For pixel-level anomaly mapping, we compute patch-level scores from the refined patch tokens $P_b \in \mathbb{R}^{(N-1) \times D}$ extracted in Appendix~\ref{app:csr}, without any spatial supervision or region masks.

\textbf{Contextual anomaly setting.} For contextual anomaly detection on CAAD, anomalies are semantic and relational rather than local or textural. The relevant signal resides in the final transformer layers, where representations capture high-level subject--context semantics. We therefore use only the final-layer patch tokens for scoring. For each patch $j$ and branch $b$:
\begin{equation}
    m_{b,j} = \text{sim}(p_{b,j}, \tilde{t}_1) - \text{sim}(p_{b,j}, \tilde{t}_0),
\end{equation}
where $p_{b,j}$ is the $j$-th normalized patch embedding. Branch patch margins are aggregated using the CRM attention weights $\alpha$ from Eq.~8:
\begin{equation}
    m_j = \alpha_s \, m_{s,j} + \alpha_c \, m_{c,j} + \alpha_g \, m_{g,j}.
\end{equation}
The patch scores $\{m_j\}$ are reshaped to the spatial token grid and bilinearly upsampled to the original image resolution, yielding a pixel-level anomaly map $M(x) \in \mathbb{R}^{H \times W}$.

\textbf{Structural anomaly setting.} For structural anomaly benchmarks such as MVTec-AD \citep{bergmann2019mvtec} and VisA \citep{zou2022spot} (Appendix~\ref{app:structural}), anomalies manifest as local appearance defects that require multi-scale spatial information. Following prior work, we aggregate patch features from multiple transformer layers rather than the final layer alone. The scoring and upsampling procedure remains the same. Details of the layer selection and the single-branch reduction used for these benchmarks are described in Appendix~\ref{app:structural}.

\section{Training Details}
\label{app:hyperparams}

\subsection{Training Protocol}
\label{app:training}

CC-CLIP is trained in two sequential stages with the CLIP backbone frozen throughout.

\textbf{Stage 1: Text refinement.} The first $L{=}3$ layers of the CLIP text encoder are adapted for 8 epochs using the disentanglement objectives (Eq.~6). All visual parameters remain frozen. This stage produces the refined normal/anomalous anchors $(\tilde{t}_0, \tilde{t}_1)$.

\textbf{Stage 2: Visual adaptation.} The CSR adapters and CRM projections are trained end-to-end for 20 epochs using the branch specialization loss $\mathcal{L}_{\text{align}}$, decorrelation loss $\mathcal{L}_{\text{decorr}}$, and fusion loss $\mathcal{L}_{\text{fuse}}$. Text refinement parameters are kept frozen in this stage. Training uses Adam with a learning rate of $2 \times 10^{-3}$ and cosine decay.

For few-shot settings, $N$ normal and $N$ anomalous images per class are sampled from CAAD-Syn. For full training, all available images are used. The same protocol is applied to all baselines for fair comparison. Full training completes in approximately 10 hours and inference on CAAD-Real in 10 minutes on a single H100 GPU.

\subsection{Hyperparameters}

Table~\ref{tab:hyperparams} summarizes the default hyperparameters used across all experiments.

\begin{table}[h]
\centering
\caption{Default hyperparameters.}
\label{tab:hyperparams}
\begin{tabular}{ll@{\hskip 48pt}ll}
\toprule
\textbf{Component} & \textbf{Setting} & \textbf{Component} & \textbf{Setting} \\
\midrule
CLIP backbone & ViT-L/14-336 & $\lambda_o$ (ortho) & 0.10 \\
Image resolution & $518 \times 518$ & $\lambda_c$ (cons) & 0.10 \\
Text learning rate & $5 \times 10^{-5}$ & $\lambda_g$ (ground) & 0.05 \\
Text epochs & 8 & $\lambda_{\text{decorr}}$ & 0.10 \\
Text batch size & 128 & $\lambda_{\text{fc}}$ (fuse-cons) & 0.5 \\
Text depth & $L = 3$ & $\lambda_{\text{fe}}$ (fuse-ent) & 0.03 \\
Image learning rate & $2 \times 10^{-3}$ & Optimizer & Adam \\
Image epochs & 20 & $(\beta_1, \beta_2)$ & $(0.5, 0.999)$ \\
Image batch size & 8 & Scheduler & MultiStepLR \\
CSR depth & $K = 6$ & Logit temperature & 100 \\
\bottomrule
\end{tabular}
\end{table}

\begin{figure}[!h]
\centering
\includegraphics[width=\textwidth]{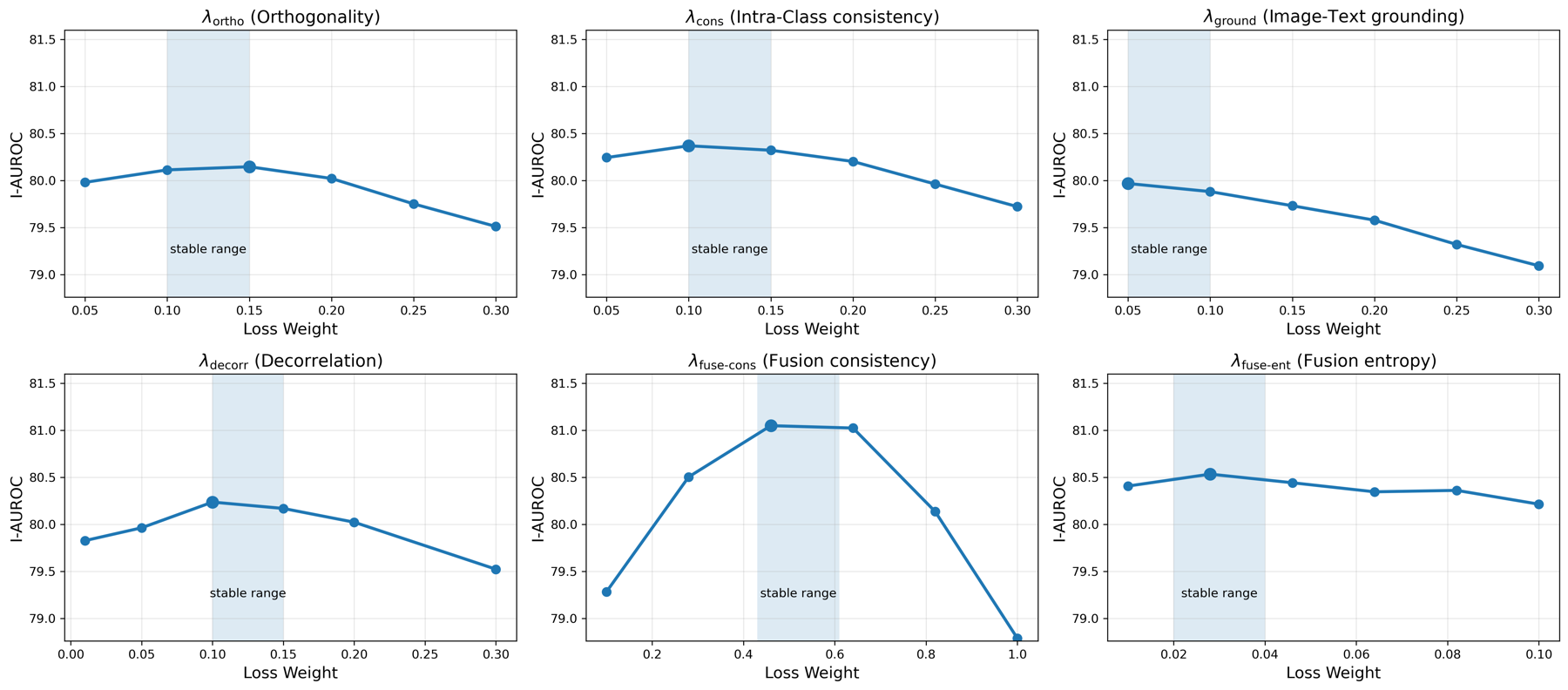}
\caption{Post-hoc loss weight sensitivity on CAAD-Real. Hyperparameters were selected on a CAAD-Syn validation split; curves confirm that chosen defaults generalize to real-world data across a broad range.}
\label{fig:sensitivity}
\end{figure}

\subsection{Loss Weight Sensitivity}
Hyperparameters were selected on a held-out CAAD-Syn validation split; 
Figure~\ref{fig:sensitivity} provides post-hoc confirmation of their stability on CAAD-Real. Each curve varies one weight while keeping all others at their default values. Across all components, performance remains stable within a broad range, with clear optima near the default values, confirming CC-CLIP is not sensitive to hyperparameter tuning.

The orthogonality weight $\lambda_o$ is annealed linearly over the first 30\% of text refinement epochs to prevent early collapse of the normal/anomalous embeddings. All other weights are fixed throughout training.

\section{Dataset Details}
\label{app:dataset}

\subsection{CAAD-Syn Construction}

CAAD-Syn is generated through a three-stage pipeline:

\textbf{Prompt generation.} Human-authored natural-language prompts describe both normal and anomalous subject--context pairings for each of the 15 categories. Each category has multiple normal and anomalous scene types to ensure no single context is exclusively associated with either label.

\textbf{Image synthesis.} Images are generated using FLUX.2-dev~\citep{blackforestlabs2025flux2dev}, a 32B-parameter rectified flow model that produces high-fidelity, artifact-free images with strong prompt adherence. This ensures that anomalies arise from semantic subject--context incompatibility rather than generative artifacts.


\textbf{Human verification.} Every image is manually inspected, requiring approximately 1--2 minutes per sample. Images with unrealistic placement, incorrect perspective, inconsistent lighting, or poor segmentation quality are discarded. Approximately 20\% of generated samples are removed through this process.

The final dataset contains 23K images across 15 object and action categories with balanced normal/anomalous labels. The automated pipeline is extensible: new categories and contexts can be added by providing additional prompts. Representative CAAD-Syn samples are shown in Appendix~\ref{app:syn_samples}. A subset of CAAD-Syn and CAAD-Real 
is available for inspection at \url{https://anonymous.4open.science/r/CC-CLIP_Dataset/ReadMe.md}; the full dataset will be released upon acceptance.

\subsection{CAAD-Real Collection}

CAAD-Real contains 1,026 real-world images (547 normal, 479 anomalous) covering the same 15 categories. Images are sourced from Pexels and Unsplash under free-use licenses. Collection involved exhaustive search across both platforms. Contextually anomalous images are inherently rare in natural photo collections, making the 479 anomalous samples a practical upper bound for honest evaluation. 

\subsection{Dataset Statistics and Comparison}

Tables~\ref{tab:dataset_comparison} and~\ref{tab:dataset_quality} contextualize CAAD within the anomaly detection landscape and quantify synthetic image quality. CLIP Score measures text-image alignment. DIRE-MSE measures proximity to the natural image manifold. CAAD-Syn matches or exceeds real-image benchmarks on both metrics.

\begin{table}[h]
\centering
\begin{minipage}[t]{0.48\textwidth}
\centering
\caption{Comparison with existing AD datasets.}
\label{tab:dataset_comparison}
\small
\setlength{\tabcolsep}{3pt}
\begin{tabular}{llccc}
\toprule
\textbf{Type} & \textbf{Dataset} & \textbf{Train} & \textbf{Test} & \textbf{Real} \\
\midrule
\multirow{3}{*}{Structural} & MVTec-AD \citep{bergmann2019mvtec} & 3.6K & 1.7K & \checkmark \\
& VisA \citep{zou2022spot} & 9.6K & 1.2K & \checkmark \\
& LOCO \citep{bergmann2022beyond} & 1.7K & 1.5K & \checkmark \\
\midrule
\multirow{2}{*}{OOC} & MIT-OOC \citep{choi2012context} & 0 & 42 & \checkmark \\
& COCO-OOC \citep{acharya2022detecting} & 30K & 2.4K & -- \\
\midrule
\multirow{2}{*}{Context.} & CAAD-Syn & 23K & -- & \\
& CAAD-Real & -- & 1.0K & \checkmark \\
\bottomrule
\end{tabular}
\end{minipage}
\hfill
\begin{minipage}[t]{0.48\textwidth}
\centering
\caption{Synthetic image quality metrics.}
\label{tab:dataset_quality}
\small
\setlength{\tabcolsep}{3pt}
\begin{tabular}{lcc}
\toprule
\textbf{Dataset} & \textbf{CLIP Score $\uparrow$} & \textbf{DIRE-MSE $\downarrow$} \\
\midrule
MIT-OOC \citep{choi2012context} & 0.29 & 0.31 \\
COCO-OOC \citep{acharya2022detecting} & 0.31 & 0.27 \\
\midrule
CAAD-Syn & \textbf{0.36} & \textbf{0.22} \\
\bottomrule
\end{tabular}
\vspace{6pt}
\end{minipage}
\end{table}
\subsection{Ethical Considerations}

CAAD-Syn is fully synthetic and contains no identifiable individuals or copyrighted imagery. All prompts are manually curated to avoid sensitive or harmful content. CAAD-Real images are sourced under free-use licenses from Pexels and Unsplash. The dataset is intended exclusively for research on visual reasoning and anomaly detection.

\section{Extended Ablations}
\label{app:ablations}

\subsection{Normal-Only Training}

Table~\ref{tab:normal_only} evaluates all methods when trained using only normal images, without exposure to anomalous examples. This setting is stricter and aligns with traditional one-class anomaly detection assumptions.

\begin{table}[h]
\centering
\caption{Normal-only training on CAAD-Syn, evaluated on CAAD-Real.}
\label{tab:normal_only}
\begin{tabular}{lcc}
\toprule
\textbf{Method} & \textbf{I-AUROC} & \textbf{I-AP} \\
\midrule
CLIP & 39.60 & 42.20 \\
WinCLIP & 50.80 & 51.30 \\
AnomalyCLIP & 52.50 & 58.20 \\
AA-CLIP & 52.80 & 56.10 \\
\midrule
CC-CLIP & \textbf{70.40} & \textbf{72.20} \\
\bottomrule
\end{tabular}
\end{table}

As expected, performance decreases for all methods compared to the balanced training setting (Table~\ref{tab:main}). However, CC-CLIP maintains a substantial margin over all baselines, indicating that the improvements stem from the architectural design of text-conditioned decomposition and compatibility reasoning rather than from access to anomalous supervision.

\subsection{Text Loss Joint Sweep}
A natural question is whether $\mathcal{L}_{\text{ortho}}$ and $\mathcal{L}_{\text{cons}}$ interact adversarially, since orthogonality pushes $\tilde{t}_0$ and $\tilde{t}_1$ apart while consistency pulls them toward a shared prototype. Table~\ref{tab:joint_sweep} shows a post-hoc joint sweep on CAAD-Real; weights were selected on a held-out CAAD-Syn 
validation split. Performance is stable along the diagonal and peaks near the default values, confirming that the two objectives are complementary: orthogonality separates contextual semantic states while consistency preserves shared class identity.

\begin{table}[h]
\centering
\begin{minipage}[t]{0.51\textwidth}
\centering
\caption{Post-hoc joint sweep of $\lambda_o$ and $\lambda_c$ on CAAD-Real.}
\label{tab:joint_sweep}
\small
\setlength{\tabcolsep}{4pt}
\begin{tabular}{c|cccc}
\toprule
\diagbox{$\lambda_o$}{$\lambda_c$} & 0.05 & 0.10 & 0.15 & 0.30 \\
\midrule
0.05 & 75.60 & 76.20 & 76.60 & 73.20 \\
0.10 & 75.80 & \textbf{80.21} & 79.80 & 72.50 \\
0.15 & 77.40 & 79.00 & 78.60 & 71.80 \\
0.30 & 73.10 & 72.30 & 71.50 & 68.40 \\
\bottomrule
\end{tabular}
\end{minipage}
\hfill
\begin{minipage}[t]{0.45\textwidth}
\centering
\caption{Corruption robustness on CAAD-Real.}
\label{tab:corruption}
\small
\setlength{\tabcolsep}{4pt}
\begin{tabular}{lcc}
\toprule
\textbf{Corruption} & \textbf{I-AUROC} & \textbf{$\Delta$} \\
\midrule
None (original) & 80.21 & -- \\
JPEG (Q=10) & 80.11 & -0.10 \\
Gaussian ($\sigma$=0.1) & 79.40 & -0.81 \\
Heavy blur ($k$=15) & 77.40 & -2.81 \\
\bottomrule
\end{tabular}
\end{minipage}
\end{table}

\subsection{Corruption Robustness}

To verify that CC-CLIP relies on semantic subject--context reasoning rather than low-level visual cues, we apply three corruptions to CAAD-Real images at test time. Performance remains stable under JPEG compression and Gaussian noise, which specifically degrade or introduce low-level artifacts. The minor drop under heavy blur is expected, as blur destroys semantic object structure rather than superficial cues.

\subsection{Per-Class Results on CAAD-Real}

Table~\ref{tab:perclass} reports per-class I-AUROC on CAAD-Real under full training.

\begin{table}[h]
\centering
\caption{Per-class I-AUROC on CAAD-Real (full training). Mean $\pm$ std over 3 runs.}
\label{tab:perclass}
\setlength{\tabcolsep}{3pt}
\begin{tabular}{lc@{\hskip 12pt}lc}
\toprule
\textbf{Class} & \textbf{I-AUROC} & \textbf{Class} & \textbf{I-AUROC} \\
\midrule
Aeroplane & 91.20 $\pm$ 0.8 & Person cycling & 90.50 $\pm$ 1.1 \\
Animal & 79.50 $\pm$ 1.3 & Person riding horse & 71.80 $\pm$ 1.5 \\
Boat & 72.40 $\pm$ 1.2 & Person running & 70.50 $\pm$ 1.4 \\
Camping tent & 71.20 $\pm$ 1.6 & Person skateboarding & 83.10 $\pm$ 0.9 \\
Car parked & 87.30 $\pm$ 0.7 & Person with umbrella & 74.80 $\pm$ 1.3 \\
Child playing & 73.80 $\pm$ 1.4 & Shopping cart & 96.20 $\pm$ 0.5 \\
Fire & 94.10 $\pm$ 0.6 & Sofa & 94.80 $\pm$ 0.4 \\
Motorbike & 72.10 $\pm$ 1.5 & & \\
\midrule
\textbf{Overall} & \textbf{80.21 $\pm$ 0.5} & & \\
\bottomrule
\end{tabular}
\end{table}

\subsection{VLM Baseline Prompting Variants}
\label{app:vlm_prompts}
Table~\ref{tab:vlm_prompts} reports I-AUROC for Qwen2-VL and InternVL under two 
zero-shot prompting strategies. All prompts elicit a continuous normality score (0--10) 
to enable AUROC computation. The exact prompt strings are as follows:
\begin{itemize}
    \item \textbf{P1:} ``Rate from 0 to 10 how normal this scene is, where 0 means 
    completely anomalous and 10 means perfectly normal. Reply with only a number.''
    \item \textbf{P2:} ``Rate this image from 0 to 10 based on how typical or expected 
    the scene appears. Reply with only a number.''
\end{itemize}
We note that this comparison is not strictly controlled for either side: CC-CLIP trains on synthetic CAAD-Syn data whereas VLMs are evaluated zero-shot, but VLMs are pretrained on web-scale corpora that may include images similar to the web-sourced CAAD-Real images, providing an implicit data advantage. That CC-CLIP outperforms VLMs despite this suggests the benefit of explicit subject--context decomposition over general-purpose reasoning.
\begin{table}[h]
\centering
\caption{VLM performance under different zero-shot prompting strategies on CAAD-Real. CC-CLIP shown for reference.}
\label{tab:vlm_prompts}
\setlength{\tabcolsep}{6pt}
\begin{tabular}{llc}
\toprule
\textbf{Prompt} & \textbf{Model} & \textbf{I-AUROC} \\
\midrule
\multirow{2}{*}{P1: Normality rating} 
& Qwen2-VL & 73.12 \\
& InternVL  & 44.53 \\
\midrule
\multirow{2}{*}{P2: Typicality rating} 
& Qwen2-VL & 69.52 \\
& InternVL  & 52.23 \\
\midrule
CC-CLIP (Ours) & --  & \textbf{80.21} \\
\bottomrule
\end{tabular}
\end{table}

CC-CLIP surpasses all VLM prompt variants despite being orders of magnitude smaller in model size.





\section{Structural Anomaly Detection Results}
\label{app:structural}

\subsection{Single-Branch Configuration}

For structural anomaly benchmarks such as MVTec-AD~\citep{bergmann2019mvtec} and VisA~\citep{zou2022spot}, anomalies are local appearance defects with no meaningful subject--context distinction. Since contextual decomposition is unnecessary in this setting, we use only the global branch with a single CSR adapter, reducing CC-CLIP to a standard single-stream architecture. Pixel-level scoring follows standard practice by aggregating patch features from multiple transformer layers~\citep{zhou2023anomalyclip, ma2025aaclip}, as structural defects require multi-scale spatial information.

\subsection{Results}

Table~\ref{tab:structural} reports zero-shot cross-dataset transfer following the protocol of AA-CLIP \citep{ma2025aaclip}: the model is trained on the full split of one dataset and evaluated on the other without any target-domain supervision. Class information is provided only through text prompts at inference.

\begin{table}[h]
\centering
\caption{Zero-shot I-AUROC and P-AUROC on MVTec-AD and VisA. CC-CLIP uses a single-branch configuration. \textbf{Best} and \underline{second best} results highlighted.}
\label{tab:structural}
\setlength{\tabcolsep}{4pt}
\begin{tabular}{llcccc}
\toprule
& & \multicolumn{2}{c}{\textbf{MVTec-AD}} & \multicolumn{2}{c}{\textbf{VisA}} \\
\cmidrule(lr){3-4} \cmidrule(lr){5-6}
\textbf{Method} & \textbf{Venue} & I-AUROC & P-AUROC & I-AUROC & P-AUROC \\
\midrule
CLIP & ICML 2021 & 86.1 & 38.4 & 66.4 & 46.6 \\
WinCLIP & CVPR 2023 & 91.8 & 85.1 & 78.0 & 79.6 \\
MVFA-AD & CVPR 2024 & 86.6 & 84.9 & 76.5 & 93.4 \\
AnomalyCLIP & ICLR 2024 & 90.9 & 91.1 & 82.1 & 95.4 \\
AdaCLIP & ECCV 2024 & 90.0 & 89.9 & 84.3 & 95.5 \\
AA-CLIP & CVPR 2025 & \underline{92.0} & \underline{91.9} & \underline{84.6} & \underline{95.5} \\
\midrule
CC-CLIP & -- & \textbf{94.2} & \textbf{92.8} & \textbf{84.9} & \textbf{95.5} \\
\bottomrule
\end{tabular}
\end{table}

CC-CLIP achieves state-of-the-art performance on MVTec-AD \citep{bergmann2019mvtec} and matches top results on VisA \citep{zou2022spot} in this single-branch configuration. This confirms that the text refinement and disentanglement objectives improve CLIP's feature quality independently of the multi-branch contextual reasoning, and that CC-CLIP does not sacrifice structural anomaly detection capability.

\section{Broader Impact}
\label{app:impact}
Contextual anomaly detection has positive applications in safety-critical domains such as surveillance, autonomous driving, and industrial monitoring, where detecting semantically inappropriate objects or actions in a scene can prevent hazardous situations. However, the notion of what constitutes a ``normal'' subject--context pairing is inherently culturally and geographically situated. Additionally, CC-CLIP inherits CLIP's pretraining biases, which may associate certain subjects with culturally specific contexts (e.g., associating particular activities with specific environments), potentially propagating skewed normality judgments. Deploying such systems without careful calibration could lead to false positives in unfamiliar but legitimate contexts, disproportionately affecting underrepresented environments. We encourage practitioners to validate on 
domain-specific data before deployment.

\section{Additional Qualitative Results}
\label{app:qualitative}

\subsection{CRM Attention for Contextual Reasoning}

\begin{figure}[!h]
\centering
\includegraphics[width=1\textwidth]{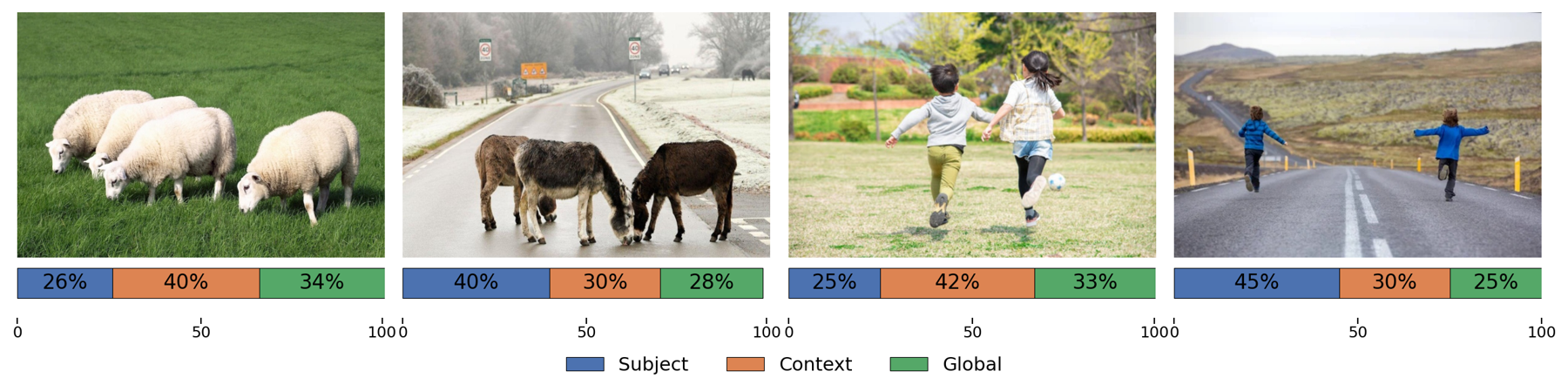}
\caption{CRM branch attention for the same subject in normal (left) and anomalous (right) contexts. Attention shifts toward the subject branch when contextual incompatibility is detected.}
\label{fig:crm_attention}
\end{figure}

Figure~\ref{fig:crm_attention} visualizes CRM branch attention for same-category subjects in normal and anomalous contexts. For compatible scenes, the model distributes attention across context and global branches, reflecting contextual support. For incompatible scenes, attention shifts toward the subject branch, signaling subject--context mismatch. This redistribution confirms that anomaly decisions arise from conditional compatibility reasoning rather than appearance alone.

\subsection{Failure Cases}

Figure~\ref{fig:failures} shows representative failure modes on CAAD-Real.

\textbf{False positive (rooftop tent on vehicle).} The model flags a rooftop tent mounted on an off-road vehicle as anomalous, despite this being a common real-world configuration. This suggests the model conflates distributional novelty with contextual incompatibility, highlighting a limitation in training coverage rather than reasoning capability.

\textbf{Mislocalization (skateboarders).} The heatmap focuses on the visually salient foreground figure rather than the semantically anomalous elements. This indicates that the model occasionally defaults to visual saliency over contextual reasoning, particularly in cluttered scenes with multiple competing subjects.

\begin{figure}[!h]
\centering
\includegraphics[width=0.8\textwidth]{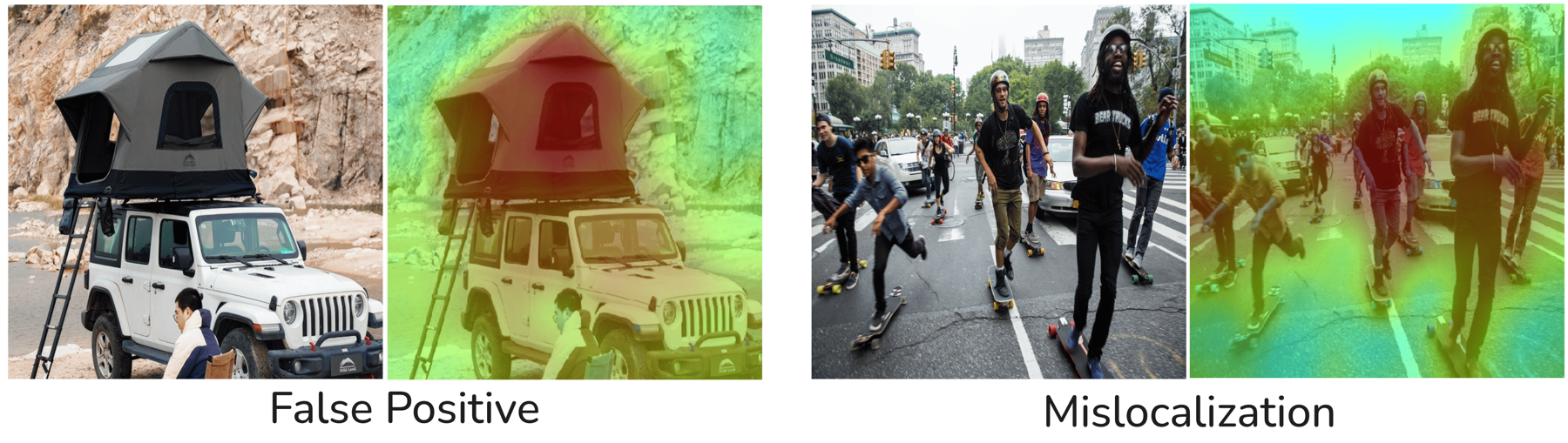}
\caption{Failure cases on CAAD-Real. Left: rare combination unseen during training. Right: ambiguous context where subject compatibility is subjective.}
\label{fig:failures}
\end{figure}

\subsection{CAAD-Syn Sample Images}
\label{app:syn_samples}

Figure~\ref{fig:syn_samples} shows representative normal and anomalous images from CAAD-Syn across four categories, illustrating the semantic subject--context incompatibility that defines contextual anomalies.

\begin{figure}[h]
    \centering
    \includegraphics[width=\linewidth]{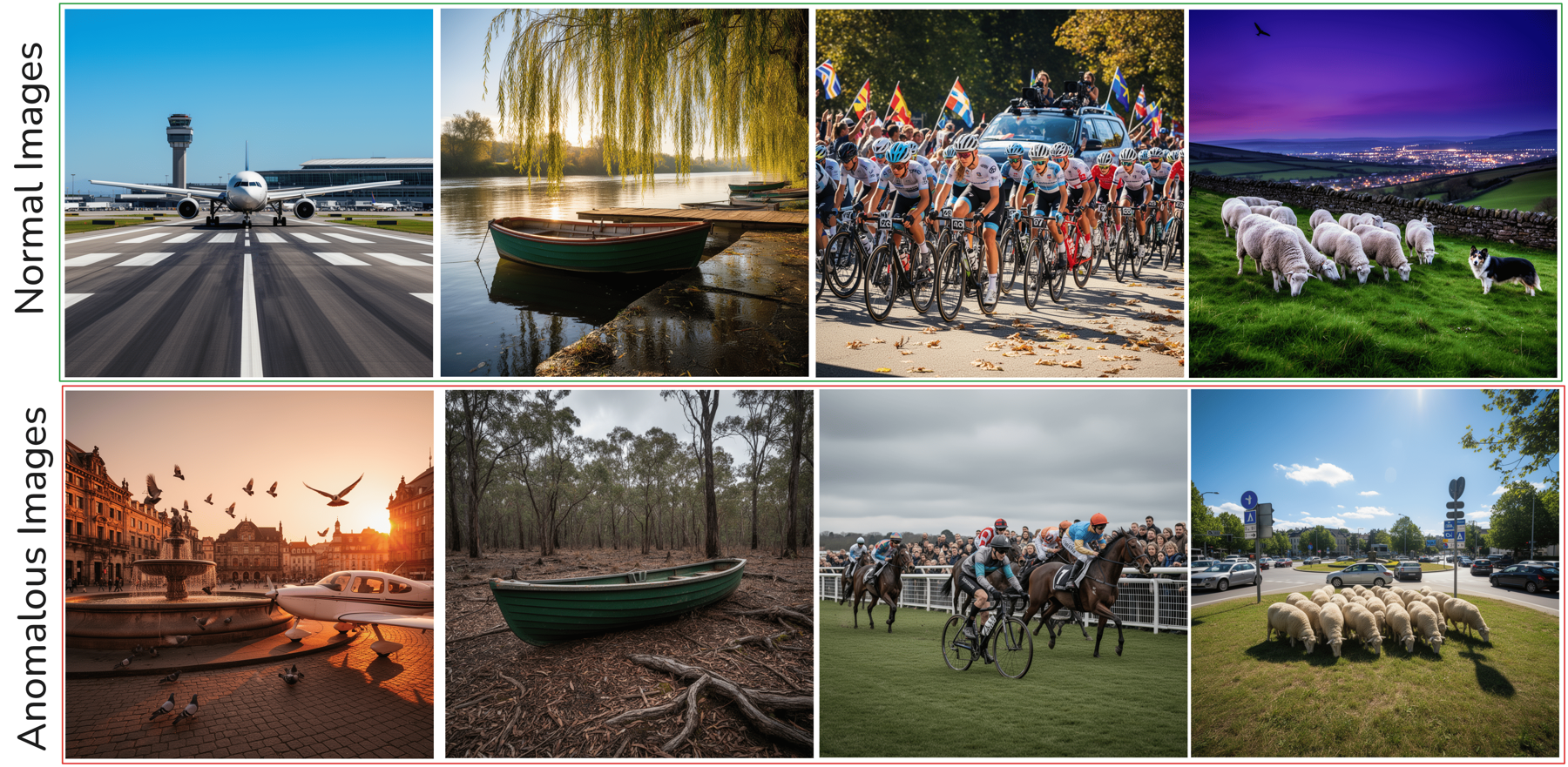}
    \caption{Representative normal (top) and anomalous (bottom) images from CAAD-Syn 
    across four categories. Anomalies arise from semantic subject--context incompatibility 
    rather than appearance-level defects. All images are synthetically generated using 
    FLUX.2-dev~\citep{blackforestlabs2025flux2dev}.}
    \label{fig:syn_samples}
\end{figure}

\subsection{Class-Wise Anomaly Heatmaps on CAAD-Real}

Figure~\ref{fig:classwise_heatmaps} shows anomaly heatmaps for all 15 categories in CAAD-Real. CC-CLIP consistently localizes the contextually incompatible subject across diverse object types, actions, and scene contexts.

\begin{figure}[h]
\centering
\includegraphics[width=\textwidth]{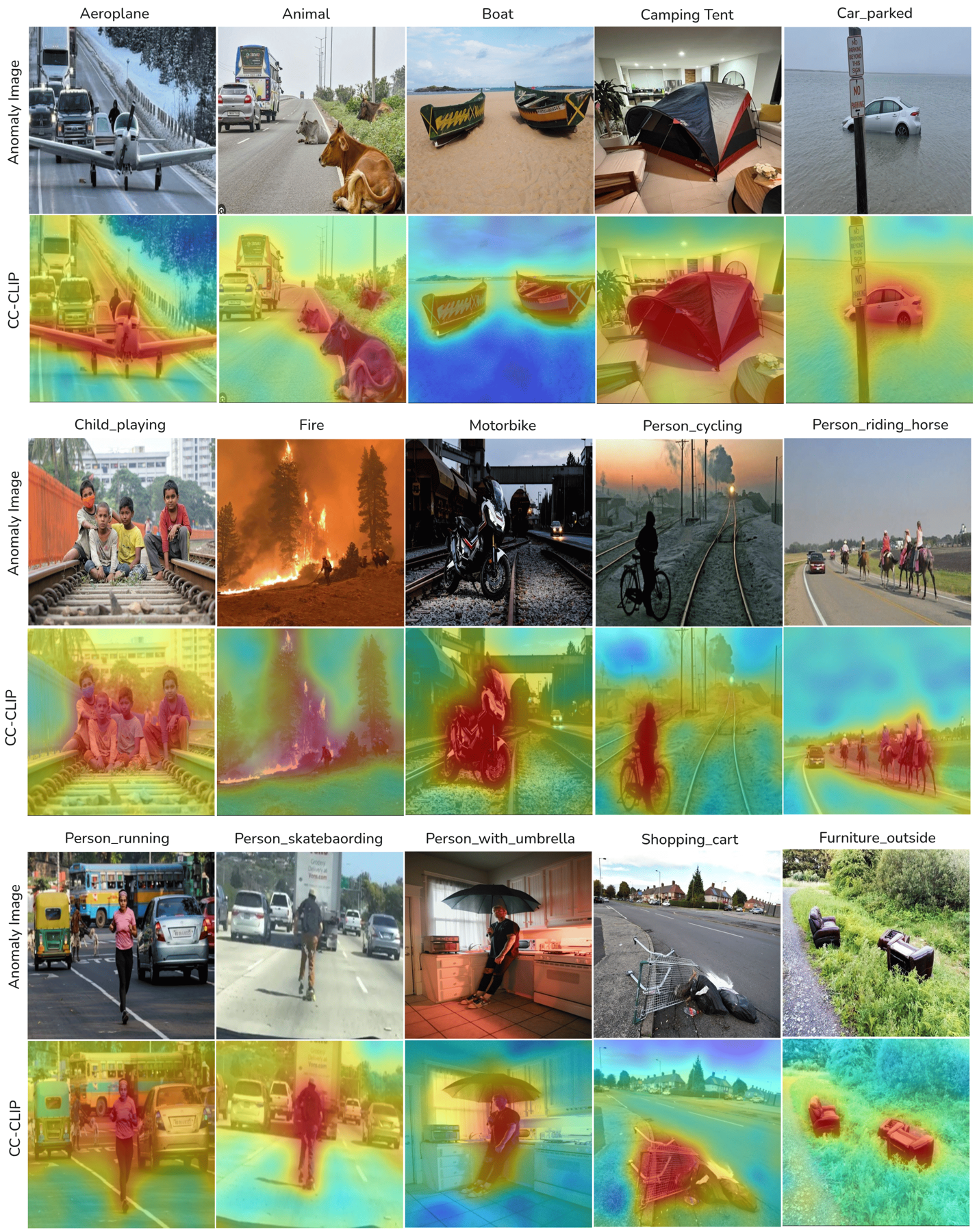}
\caption{Class-wise anomaly heatmaps on CAAD-Real. For each class: input image (top) and CC-CLIP heatmap (bottom). The model consistently highlights the subject responsible for the contextual violation.}
\label{fig:classwise_heatmaps}
\end{figure}

\subsection{Structural Anomaly Localization}

Figures~\ref{fig:mvtec_qual} and~\ref{fig:visa_qual} show zero-shot pixel-level anomaly maps on MVTec-AD and VisA respectively. In single-branch mode, CC-CLIP produces precise defect localization maps despite receiving no supervision from the target dataset.

\begin{figure}[!th]
\centering
\includegraphics[width=\textwidth]{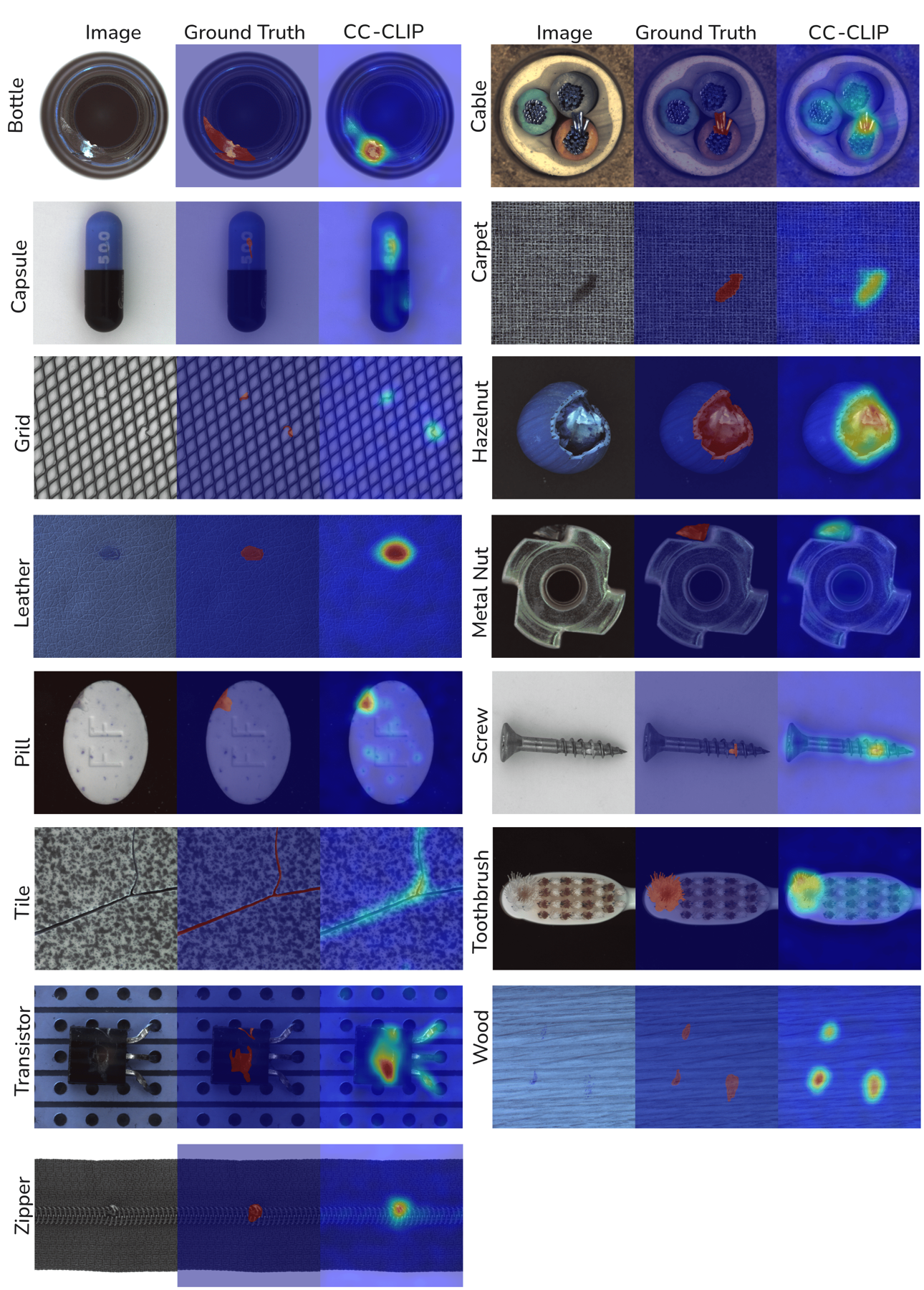}
\caption{Zero-shot anomaly localization on MVTec-AD. CC-CLIP is trained on VisA and evaluated without any MVTec-AD supervision.}
\label{fig:mvtec_qual}
\end{figure}

\begin{figure}[!ht]
\centering
\includegraphics[width=\textwidth]{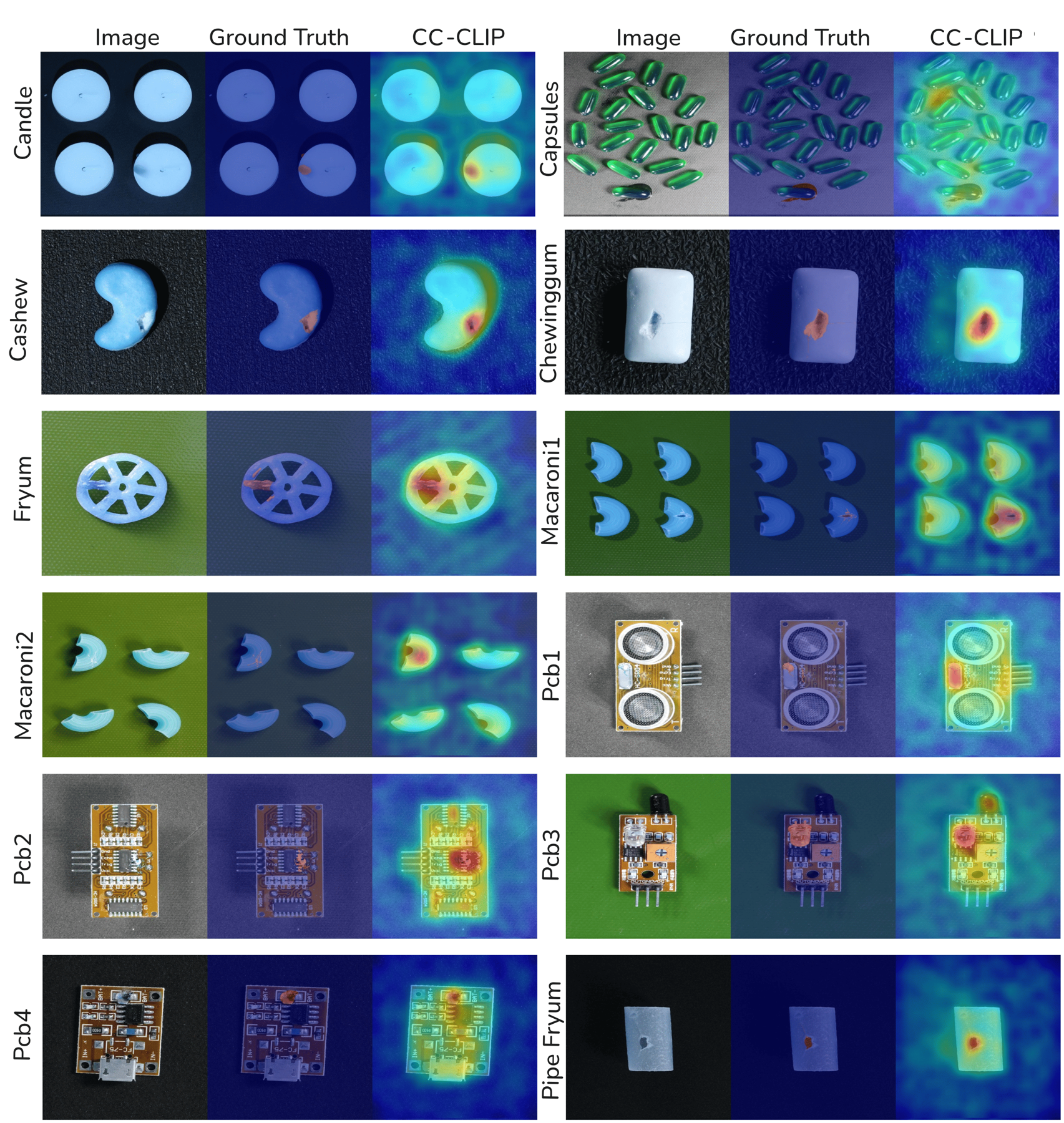}
\caption{Zero-shot anomaly localization on VisA. CC-CLIP is trained on MVTec-AD and evaluated without any VisA supervision.}
\label{fig:visa_qual}
\end{figure}

\clearpage
\newpage

\end{document}